\documentclass{aamas2015-no-copy}
\usepackage[T1]{fontenc}
\usepackage[letterpaper]{geometry}
\usepackage{verbatim}
\usepackage{booktabs}
\usepackage{amsmath}
\usepackage{graphicx}

\makeatletter

\providecommand{\tabularnewline}{\\}

\toappear{}

\usepackage{enumitem}
\setitemize[0]{leftmargin=15pt,itemindent=0pt,topsep=3pt,itemsep=2pt,partopsep=1pt,parsep=1pt}
\setenumerate[0]{leftmargin=15pt,itemindent=0pt,topsep=3pt,itemsep=1pt,partopsep=0pt,parsep=0pt}

\providecommand{\citeauthor}[1]{\citeA{#1}}
\newcommand{\citet}[1]{\citeauthor{#1} \shortcite{#1}}
\newcommand{\citep}[1]{\cite{#1}}

\usepackage{psfrag}
\usepackage{graphicx}
\usepackage{amsmath}
\usepackage{amssymb}
\usepackage{amsthm}
\usepackage{amsbsy}
\usepackage{longtable}
\usepackage{verbatim}
\usepackage{icomma} 
\usepackage{ifthen}

\usepackage[shortcuts]{extdash}

\usepackage{subfigure}

\graphicspath{{.},{figs/}}

\renewcommand{\newdef}[2]{\newtheorem{#1}{#2}}

\newtheorem{theorem}{Theorem}
\newtheorem{lemma}{Lemma}
\newtheorem{corollary}{Corollary}
\theoremstyle{definition}
\newdef{definition}{Definition} 
\newdef{observation}{Observation} 
\newdef{example}{Example} 

\providecommand{\joint}[1]{#1}

\providecommand{\argsA}[2]{ {#1}_{#2} } 
\providecommand{\argsG}[2]{ {\joint{#1}}_{#2} } 
\providecommand{\argsT}[2]{ {#1}^{#2} } 
\providecommand{\argsI}[2]{ {#1}^{#2} } 
\providecommand{\argsAI}[3]{ {#1}_{#2}^{#3} }
\providecommand{\argsAT}[3]{ {#1}_{#2}^{#3} }
\providecommand{\argsGT}[3]{ {\joint{#1}}_{#2}^{#3} } 
 
\providecommand{\argsIT}[3]{ {#1}^{#2,#3} }
\providecommand{\argsAIT}[4]{ {#1}_{#2}^{#3,#4} }

\providecommand{\agNONE}{} 
\providecommand{\hor}{h}

\providecommand{\V}{V}
\providecommand{\Q}{Q}
\providecommand{\RA}[1]     {\argsA{R}{#1}}
\providecommand{\QT}[1]     {\argsT{Q}{#1}}

\providecommand{\IndF}      { \boldsymbol{1} }
\providecommand{\IndFo}[1]  { \IndF_{ \{ #1 \} } }              

\providecommand{\lbA}[1]    {\argsAT{b}{#1}{l}}

\providecommand{\gbA}[1]    {\argsAT{b}{#1}{g}}

\providecommand{\ptss}          {\sigma}

\providecommand{\ptssIT}[2]     {\argsIT{\ptss}{#1}{#2}}

\providecommand{\NMF}        {y}                                 
\providecommand{\MF}         {x}                                 
\providecommand{\MFn}        {xn}                               
\providecommand{\MFl}        {xl}                               
\providecommand{\PF}         {xPRIVATE}                              

\providecommand{\mfIT}[2]   {\argsIT{\MF}{#1}{#2}}

\providecommand{\mfAIT}[3]  {\argsAIT{\MF}{#1}{#2}{#3}}
    
\providecommand{\mfAT}[2]   {\argsAT{\joint{\MF}}{#1}{#2}}
\providecommand{\mfl}[1]    {{\MFl}}

\providecommand{\mflAT}[2]  {\argsAT{\joint{\MFl}}{#1}{#2}}
\providecommand{\mfn}[1]    {{\MFn}}

\providecommand{\mfnAT}[2]   {\argsAT{\joint{\MFn}}{#1}{#2}}

\providecommand{\pfA}[1]    {\argsA {\joint{\PF}}{#1}}    
\providecommand{\pfAT}[2]   {\argsAT{\joint{\PF}}{#1}{#2}}

\providecommand{\nmfA}[1]   {\argsA{\joint{\NMF}}{#1}}
\providecommand{\nmfAT}[2]  {\argsAT{\joint{\NMF}}{#1}{#2}}

\providecommand{\INFL}        {I} 

\providecommand{\iffunc}    {\INFL}                               

\providecommand{\ifpiAT}[2] {\argsAT{\INFL}{\rightarrow#1}{#2}}  
   
\providecommand{\ifsAT}[2]  {\argsAT    {\joint{\IFSOURCE}}{\rightarrow #1}{#2}} 

\providecommand{\DSET}          {\joint{D}}

\providecommand{\dsetAT}[2]     {\argsAT{\DSET}{#1}{#2}}

\providecommand{\sect}{Section~}

\providecommand{\fig}{Fig.~}

\providecommand{\tab}{Table~}

\providecommand{\app}{Appendix~}
\providecommand{\lem}{Lemma~}
\providecommand{\thm}{Theorem~}

\providecommand{\cor}{Corollary~}

\providecommand{\defas}     {\operatorname{\triangleq}}

\providecommand{\inprod}[2]{{#1} \cdot {#2}}

\providecommand{\mult}{\cdot}
\providecommand{\x}{\mult}

\newcommand{\E}{\mathbf{E}}

\providecommand{\reals}{\mathbb {R} }

\providecommand{\algName}[1]{{\sc{#1}}}
\providecommand{\problemName}[1]{\textsc{#1}}

\newcommand{\set}[1]{\mathcal{#1}}
\providecommand{\joint}[1]{\boldsymbol{#1}}

\providecommand{\discount}{\gamma}

\providecommand{\ts}{t}             

\providecommand{\agentSymb}{D}
\providecommand{\agentS}{\mathcal{\agentSymb}}
\providecommand{\agentI}[1]{{#1}}
\providecommand{\nrA}{n} 

\providecommand{\excl}[1]{-{#1}}

\providecommand{\Aug}[1]        {\bar{#1}}
\providecommand{\sAugA}[1]      {\argsA{\Aug{s}}{#1}}
\providecommand{\sAugAT}[2]     {\argsAT{\Aug{s}}{#1}{#2}}

\providecommand{\sS}{\set{S}}   
\providecommand{\s}{s}  
\providecommand{\sT}[1]{\argsT{\s}{#1}}

\providecommand{\nrSF}{m} 

\providecommand{\factorSymb}{X}
\providecommand{\factorSetSymb}{\set{X}}    

\providecommand{\sfacS}     {\factorSetSymb}                    
\providecommand{\sfac}      {\factorSymb}                       
\providecommand{\sfacI}[1]  {\argsI{\factorSymb}{#1}}           
\providecommand{\sfacIT}[2] {\argsIT{\factorSymb}{#1}{#2}}      
\providecommand{\sfacGT}[2] {\argsGT{\factorSymb}{#1}{#2}}      %

\providecommand{\factorValueSymb}{x}

\providecommand{\sfacv}         {\factorValueSymb}                  
                        
\providecommand{\sfacvIS}[1]    {Dom(\sfacI{#1})}                   
\providecommand{\sfacvIT}[2]    {\argsIT{\factorValueSymb}{#1}{#2}}
     
\providecommand{\sfacvG}[1]     {\argsG{\factorValueSymb}{#1}}

\providecommand{\AC}{a}                 
\providecommand{\ACS}{\set{A}}                     
\providecommand{\aA}[1] {\argsA {\AC}   {#1}}
\providecommand{\aAS}[1]{\argsA {\ACS}  {#1}}      
\providecommand{\aAT}[2]{\argsAT{\AC}   {#1}{#2}}

\providecommand{\ja}    {\joint {\AC}}          
\providecommand{\jaS}   {\joint {\ACS}}         
\providecommand{\jaT}[1]{\argsT {\ja}   {#1}}   
\providecommand{\jaG}[1]    {\argsG{\AC}    {#1}} 
\providecommand{\jaGT}[2]   {\argsGT{\AC}   {#1}{#2}} 
\providecommand{\jaGS}[1]   {\argsG{\ACS}   {#1}} 

\providecommand{\OB}{o}                
\providecommand{\OBS}{\set{O}}
\providecommand{\oA}[1] {\argsA {\OB}   {#1}}
\providecommand{\oAS}[1]{\argsA {\OBS}  {#1}}      
\providecommand{\oAT}[2]{\argsAT{\OB}   {#1}{#2}}

\providecommand{\jo}        {\joint {\OB}}      
\providecommand{\joS}       {\joint {\OBS}}    
\providecommand{\joT}[1]    {\argsT {\jo}   {#1}}
\providecommand{\joG}[1]    {\argsG {\OB}   {#1}}      
\providecommand{\joGS}[1]   {\argsG {\OBS}   {#1}}      
\providecommand{\joGT}[2]   {\argsGT{\OB}   {#1}{#2}}      

\providecommand{\Tfunc}     {T}
\providecommand{\Ofunc}     {O}

\providecommand{\REWF}{R}   
\providecommand{\R}     {\REWF}                     
\providecommand{\RA}[1] {\argsA {\REWF}   {#1}}     

\providecommand{\RS}    {\set{R}} 
\providecommand{\RI}[1] {\argsI {\REWF} {#1}}       
\providecommand{\nrR}   {\rho}                      

\providecommand{\bSymbol}           {{b}}
\providecommand{\jointBeliefSymbol} {\bSymbol} 
\providecommand{\jb}                {\jointBeliefSymbol}

\providecommand{\jbT}[1]            {\argsT{\jb}{#1}}

\providecommand{\bO}                {\argsT\bSymbol{0}} 
\providecommand{\bA}[1]             {\argsA\bSymbol{#1}}

\providecommand{\gbA}[1]             {\argsAT\bSymbol{#1}{g}}
\providecommand{\AOH}{\vec{\theta}}

\providecommand{\aoHistAT}[2]   {\argsAT  {\AOH}{#1}{\,#2}}

\providecommand{\aoHistGT}[2]   {\JAOH\argsGT{}{#1}{#2}}

\providecommand{\OH}{\vec{o}}

\providecommand{\JOH}{\vec{\joint{o}}{}}

\providecommand{\oHistT}[1]    {\JOH\argsT{}{\;#1}}

\providecommand{\oHistAT}[2]   {\argsAT  {\OH}{#1}{\,#2}}

\providecommand{\oHistGT}[2]   {\argsGT\JOH{#1}{\;#2}}

\providecommand{\AH}{\vec{a}}

\providecommand{\aHistAT}[2]   {\argsAT  {\AH}{#1}{\,#2}}

\providecommand{\POL}{\pi}

\providecommand{\jpol}      {\joint{\POL}}	                

\providecommand{\jpolG}[1]  {\argsG     {\joint{\POL}}{#1}}	

\providecommand{\polA}[1]   {\argsA{\POL}{#1}}	
  
\providecommand{\DR}{\delta}

\providecommand{\drAT}[2]{\argsAT{\DR}{#1}{#2}}

\providecommand{\jdrGT}[2]{\argsGT{\DR}{#1}{#2}}

\providecommand{\pP}{\varphi}	

\providecommand{\pJPolGT}[2]    {\argsGT{\pJPol}{#1}{#2}}	

\providecommand{\pPolAT}[2]     {\argsAT{\pP}{#1}{#2}}

\def\<#1>{%
    \expandafter\ifx\csname<#1>\endcsname\relax
        \errmessage{abbreviation <#1> undefined!}
    \else
        \csname<#1>\endcsname
    \fi
}
\def\abbr#1#2{%
    \expandafter\def\csname<#1>\endcsname{#2}%
}

\abbr{iid}{i.i.d.}

\abbr{dectig}{\problemName{Dec-Tiger}}
\abbr{sdectig}{\problemName{Skewed Dec-Tiger}}
\abbr{bcc}{\problemName{BroadcastChannel}}
\abbr{grid}{\problemName{GridSmall}}
\abbr{boxPush}{\problemName{Cooperative Box Pushing}}
\abbr{boxPushShort}{\problemName{C.\ Box Pushing}}
\abbr{recycling}{\problemName{Recycling Robots}}
\abbr{recyclingShort}{\problemName{Recycling R.}}
\abbr{hotel}{\problemName{Hotel 1}}
\abbr{ff}{\problemName{FireFighting}}
\abbr{FFG}{\problemName{FFG}}
\abbr{FFGfull}{\problemName{\problemName{Fi\-re\-Fight\-ing\-Graph}}}
\abbr{Aloha}{\problemName{Aloha}}

\abbr{AOH}{action-observation history}
\abbr{JAOH}{joint action-observation history}

\abbr{Astar}            {\algName{A*}}
\abbr{MAA}              {\algName{MAA*}}
\abbr{GMAA}             {\algName{GMAA*}}        
\abbr{GMAAC}            {\algName{GMAA-Cluster}}
\abbr{GMAAIC}           {\algName{GMAA*-IC}}    
\abbr{GMAAICE}          {\algName{GMAA*-ICE}}   
\abbr{FSPC}             {\algName{FSPC}} 

\abbr{QMDP}{\algName{QMDP}}
\abbr{QPOMDP}{\algName{QPOMDP}}
\abbr{QBG}{\algName{QBG}}

\makeatother

\begin{document}

\title{Influence-Optimistic Local Values \\
for Multiagent Planning --- Extended Version}

\author{Frans A.~Oliehoek\\
University of Amsterdam \\
University of Liverpool\\
\email{fao@liverpool.ac.uk} \and Matthijs T.~J.~Spaan\\
Delft University of Technology\\
The Netherlands\\
\email{m.t.j.spaan@tudelft.nl} \and Stefan Witwicki\\
Swiss Federal Institute \\
of Technology (EPFL)\\
\email{stefan.witwicki@epfl.ch}}

\maketitle

\noindent

\global\long\def\discount{\gamma}
\global\long\def\ts{t}
\global\long\def\excl#1{\neq{#1}}

\global\long\def\s{s}

\global\long\def\sT#1{{\argsT{s}{#1}}}

\global\long\def\Aug#1{\bar{{#1}}}

\global\long\def\sAugAT#1#2{{\argsAT{\Aug{\s}}{#1}{#2}}}

\global\long\def\sAugA#1{{\argsA{\bar{s}}{#1}}{}}

\global\long\def\sAugAT#1#2{{\argsAT{\bar{s}}{#1}{#2}}}

\global\long\def\aA#1{{\argsA{a}{#1}}}

\global\long\def\aAT#1#2{{\argsAT{a}{#1}{#2}}}

\global\long\def\oA#1{{\argsA{o}{#1}}}

\global\long\def\oAT#1#2{{\argsAT{o}{#1}{#2}}}

\global\long\def\ja{{\joint{\AC}}}

\global\long\def\jaT#1{{\argsT{\ja}{#1}}}

\global\long\def\jaG#1{{\argsG{\AC}{#1}}}

\global\long\def\jaGT#1#2{{\argsGT{\AC}{#1}{#2}}}

\global\long\def\joT#1{{\argsT{\jo}{#1}}}

\global\long\def\joGT#1#2{{\argsGT{\OB}{#1}{#2}}}

\global\long\def\aHistAT#1#2{{\argsAT{\AH}{#1}{\,#2}}}

\global\long\def\oHistAT#1#2{{\argsAT{\OH}{#1}{\,#2}}}

\global\long\def\aoHistAT#1#2{{\argsAT{\AOH}{#1}{\,#2}}}

\global\long\def\aoHistGT#1#2{{\argsGT{\AOH}{#1}{\,#2}}}

\global\long\def\polA#1{\argsA{\POL}{#1}}

\global\long\def\drAT#1#2{\argsAT{\DR}{#1}{#2}}

\global\long\def\jdrGT#1#2{\argsAT{\joint\DR}{#1}{#2}}

\global\long\def\pPolAT#1#2{\argsAT{\pP}{#1}{#2}}

\global\long\def\pJPolGT#1#2{\argsAT{\joint\pP}{#1}{#2}}

\global\long\def\jpolG#1{\argsG{\joint{\POL}}{#1}}

\global\long\def\bA#1{\argsA{b}{#1}}

\global\long\def\bO{{b^{0}}}

\global\long\def\lbA#1{\argsAT{b}{#1}{l}}

\global\long\def\gbA#1{\argsAT{b}{#1}{g}}

\global\long\def\RA#1{\argsA{R}{#1}}

\global\long\def\V{V}

\global\long\def\Q{Q}

\global\long\def\QT#1{\argsT{Q}{#1}}

\global\long\def\sfacS{\factorSetSymb}

\global\long\def\sfacIT#1#2{\argsIT{\factorSymb}{#1}{#2}}

\global\long\def\sfacGT#1#2{\argsGT{\factorSymb}{#1}{#2}}

\global\long\def\factorValueSymb{x}

\global\long\def\sfacvIT#1#2{\argsIT{\factorValueSymb}{#1}{#2}}

\global\long\def\sfacGT#1#2{\argsGT{\factorValueSymb}{#1}{#2}}

\global\long\def\mfAT#1#2{{\argsAT{\joint\MF}{#1}{#2}}}

\global\long\def\mfAIT#1#2#3{{\argsAIT{\MF}{#1}{#2}{#3}}}

\global\long\def\mfIT#1#2{{\argsAT{\MF}{#1}{#2}}}

\global\long\def\mflAT#1#2{{\argsAT{\joint\MFl}{#1}{#2}}}

\global\long\def\mflAIT#1#2#3{{\argsAIT{\MFl}{#1}{#2}{#3}}}

\global\long\def\mfnAT#1#2{{\argsAT{\joint\MFn}{#1}{#2}}}

\global\long\def\mfnAIT#1#2#3{{\argsAIT{\MFn}{#1}{#2}{#3}}}

\global\long\def\nmfA#1{\argsA{\joint\NMF}{#1}}

\global\long\def\nmfAT#1#2{\argsAT{\joint\NMF}{#1}{#2}}

\global\long\def\nmfAIT#1#2#3{\argsAIT{\NMF}{#1}{#2}#3}

\global\long\def\pfA#1{\argsA{\joint\PF}{#1}}

\global\long\def\pfAT#1#2{\argsAT{\joint\PF}{#1}{#2}}

\global\long\def\dsetAT#1#2{\argsAT{\DSET}{#1}{#2}}

\global\long\def\ifsAT#1#2{\argsAT{u}{#1}{#2}}

\global\long\def\ifpiAT#1#2{\argsAT{I}{\rightarrow#1}{#2}}

\global\long\def\co{c}

\global\long\def\argsC#1#2{\argsA{#1}{#2}}

\newcommand{\argsCT}[3]	 {\argsAT{#1}{#2}{#3}}
\newcommand{\argsCI}[3]	 {\argsAI{#1}{#2}{#3}}
\newcommand{\argsCIT}[4]	{\argsAIT{#1}{#2}{#3}{#4}}

\global\long\def\RC#1{\argsC R{#1}}

\global\long\def\RCT#1#2{\argsCT R{#1}{#2}}

\global\long\def\jaC#1{\argsC{\AC}{#1}}

\global\long\def\jaCT#1#2{\argsCT{\AC}{#1}{#2}}

\global\long\def\mfC#1{\argsC{\joint{\MF}}{#1}}

\global\long\def\mfCT#1#2{{\argsCT{\joint{\MF}}{#1}{#2}}}

\global\long\def\mfCI#1#2{{\argsCI{\MF}{#1}{#2}{}}}

\global\long\def\mfCIT#1#2#3{{\argsCIT{\MF}{#1}{#2}{#3}}}

\global\long\def\mfIT#1#2{{\argsCT{\MF}{#1}{#2}}}

\global\long\def\mflCT#1#2{{\argsCT{\joint\MFl}{#1}{#2}}}

\global\long\def\mflCI#1#2{{\argsCI{\MFl}{#1}{#2}}}

\global\long\def\mflCIT#1#2#3{{\argsCIT{\MFl}{#1}{#2}{#3}}}

\global\long\def\mfnCT#1#2{{\argsCT{\joint\MFn}{#1}{#2}}}

\global\long\def\mfnCI#1#2{{\argsCI{\MFn}{#1}{#2}}}

\global\long\def\mfnCIT#1#2#3{{\argsCIT{\MFn}{#1}{#2}{#3}}}

\global\long\def\nmfC#1{\argsC{\joint\NMF}{#1}}

\global\long\def\nmfCT#1#2{\argsCT{\joint\NMF}{#1}{#2}}

\global\long\def\nmfCIT#1#2#3{\argsCIT{\NMF}{#1}{#2}#3}

\global\long\def\pfC#1{\argsC{\joint\PF}{#1}}

\global\long\def\pfCT#1#2{\argsCT{\joint\PF}{#1}{#2}}

\global\long\def\dsetCT#1#2{\argsCT{\DSET}{#1}{#2}}

\global\long\def\ifsCT#1#2{\argsCT{u}{#1}{#2}}

\global\long\def\ifpiC#1{\argsC I{\rightarrow#1}}

\global\long\def\ifpiCT#1#2{\argsCT{I}{\rightarrow#1}{#2}}

\global\long\def\IASPstateT#1{\argsCT{\Aug{\sfacv}}{\co}{#1}}

\global\long\def\IASPbelT#1{\argsCT{\Aug{\jb}}{\co}{#1}}

\global\long\def\IASPbelIT#1#2{\argsCT{\Aug{\jb}}{#1}{#2}}

\global\long\def\PT#1{{\check{#1}}}

\global\long\def\PTIASPstateT#1{\argsT{\PT{\s}}{#1}}

\global\long\def\PTaT#1{\argsT{\PT{\aA{\agNONE}}}{#1}}

\renewcommand{\nrA}{|\agentS|} 

\newcommand{\nrASP}{|\agentS'|}

\newcommand{\PRIME}{\prime}

\newcommand{\NOPR}{}


\newcommand{\DPRIME}{}
\newcommand{\DPRIPRI}{\prime}
\newcommand{\SPRIME}{}
\begin{abstract}
\small{}Recent years have seen the development of methods for multiagent
planning under uncertainty that scale to tens or even hundreds of
agents. However, most of these methods either make restrictive assumptions
on the problem domain, or provide approximate solutions \emph{without}
any guarantees on quality. Methods in the former category typically
build on heuristic search using upper bounds on the value function.
Unfortunately, no techniques exist to compute such upper bounds for
problems with non-factored value functions. To allow for meaningful
benchmarking through measurable quality guarantees on a very general
class of problems, this paper introduces a family of \emph{influence-optimistic
}upper bounds for factored decentralized partially observable Markov
decision processes (Dec-POMDPs) that do not have factored value functions.
Intuitively, we derive bounds on very large multiagent planning problems
by subdividing them in sub-problems, and  at each of these sub-problems making
optimistic assumptions with respect to the influence that will be
exerted by the rest of the system. We numerically compare the different
upper bounds and demonstrate how we can achieve a non-trivial guarantee
that a heuristic solution for problems with hundreds of agents is
close to optimal. Furthermore, we provide evidence that the upper
bounds may improve the effectiveness of heuristic influence search,
and discuss further potential applications to multiagent planning.
\end{abstract}

\section{Introduction}

Planning for multiagent systems (MASs) under uncertainty is an important
research problem in artificial intelligence. The decentralized partially
observable Markov decision process (Dec-POMDP) is a general principled
framework for addressing such problems. Many recent approaches to
solving Dec-POMDPs propose to exploit \emph{locality of interaction}\ \cite{Nair05AAAI}
also referred to as \emph{value factorization}~\cite{Kumar11IJCAI}.
However, without making very strong assumptions, such as transition
and observation independence~\cite{Becker03AAMAS}, there is no strict
locality: in general the actions of any agent may affect the rewards
received in a different part of the system, even if that agent and
the origin of that reward are (spatially) far apart. For instance,
in a traffic network the actions taken in one part of the network
will eventually influence the rest of the network~\cite{Oliehoek08AAMAS}.

A number of approaches have been proposed to generate solutions for
large MASs \cite{Velagapudi11AAMAS,Yin11IJCAI,Oliehoek13AAMAS,Wu13IJCAI,Dibangoye14AAMAS,Varakantham14AAAI}.
However, these heuristic methods come without guarantees. In fact,
since it has been shown that approximation (given some~$\epsilon$,
finding a solution with value within $\epsilon$ of optimal) of Dec-POMDPs
is NEXP-complete~\cite{Rabinovich03AAMAS}, it is unrealistic to
expect to find general, scalable methods that have such guarantees.
However, the lack of guarantees also makes it difficult to meaningfully
interpret the results produced by heuristic methods. In this work,
we mitigate this issue by proposing a novel set of techniques that
can be used to provide upper bounds on the performance of large \emph{factored}
Dec-POMDPs.

More generally, the ability to compute upper bounds is important for
numerous reasons: 1) As stated above, they are crucial for a meaningful
interpretation of the quality of heuristic methods. 2) Such knowledge
of performance gaps is crucial for researchers to direct their focus
to promising areas. 3) Such knowledge is also crucial for understanding
which problems seem simpler to approximate than others, which in turn
may lead to improved theoretical understanding of different problems.
4) Knowledge about the performance gap of the leading heuristic methods
can also accelerate their real-world deployment, e.g., when their
performance gap is proven to be small over sampled domain instances,
or when the selection of which heuristic method to deploy is facilitated
by clarifying the trade-off of computation and closeness to optimality.
5) Upper bounds on achievable value without communication may guide
decisions on investments in communication infrastructure.\textbf{
}6) Last, but not least, these upper bounds can directly be used in
current and future heuristic search methods, as we will discuss in
some more detail at the end of this paper.

Computing upper bounds on the achievable value of a planning problem
typically involves relaxing the original problem by making some optimistic
assumptions. For instance, in the case of Dec-POMDPs typical assumptions
are that the agents can communicate or observe the true state of the
system~\cite{Emery-Montemerlo04AAMAS,Szer05UAI_MAA,Roth05AAMAS,Oliehoek08JAIR}.
By exploiting the fact that transition and observation dependence
leads to a value function that is \emph{additively factored} into
a number of small components (we say that the value function is `factored',
or that the setting exhibits `value factorization'), such techniques
have been extended to compute upper bounds for so-called \emph{network-distributed
POMDPs (ND-POMDPs)} with many agents. This has greatly increased the
size of the problems that can be solved~\cite{Varakantham07AAMAS,Marecki08AAMAS,Dibangoye14AAMAS}.
Unfortunately, assuming both transition and observation independence
(or, more generally, value factorization) narrows down the applicability
of the model, and no techniques for computing upper bounds for more
general factored Dec-POMDPs with many agents are currently known.

We address this problem by proposing a general technique for computing
what we call \emph{influence-optimistic }upper bounds\emph{.} These
are upper bounds on the achievable value in large-scale MASs formed
by computing \emph{local }influence-optimistic\emph{ }upper bounds
on the value of sub-problems that consist of small subsets of agents
and state factors. The key idea is that if we make optimistic assumptions
about how the rest of the system will influence a sub-problem, we
can decouple it and effectively compute a local upper bound on the
achievable value. Finally, we show how these local bounds can be combined
into a \emph{global} upper bound. In this way, the major contribution
of this paper is that it shows how we can compute \emph{factored upper
bounds }for models that \emph{do} \emph{not admit factored value functions}.

We empirically evaluate the utility of influence-opti\-mistic upper
bounds by investigating the quality guarantees they provide for heuristic
methods, and by examining their application in a heuristic search
method. The results show that the proposed bounds are tight enough
to give meaningful quality guarantees for the heuristic solutions
for factored Dec-POMDPs with hundreds of agents.%
\footnote{In the paper, we use the word `tight' for its (empirical) meaning
of ``close to optimal'', not for its (theoretical CS) meaning of
``coinciding with the best possible bound''.%
} This is a major accomplishment since previous approaches that provide
guarantees 1) have required very particular structure such as transition
and observation independence \cite{Becker03AAMAS,Becker04AAMAS,Varakantham07AAMAS,Dibangoye14AAMAS}
or `transition-decoupledness' combined with very specific interaction
structures (transitions of an agent can be affected in a \emph{directed }fashion
and only by a small subset of other agents) \cite{Witwicki11PhD}, and 2) have not scaled beyond 50 agents. In
contrast, this paper demonstrates quality bounds in settings of hundreds
of agents that all influence each other via their actions.\textbf{}

This paper is organized as follows. First, \sect\ref{sec:background}
describes the required background by introducing the factored Dec-POMDP
model. Next, \sect\ref{sec:sub-problems} describes the sub-problems
that form the basis of our decomposition scheme. \sect\ref{sec:local-upper-bounds}
proposes local influence-optimistic upper bounds for such sub-problems
together with the techniques to compute them. Subsequently, \sect\ref{sec:global-upper-bounds}
discusses how these local upper bounds can be combined into a global
upper bound for large problems with many agents. \sect\ref{sec:evaluation}
empirically investigates the merits of the proposed bounds. \sect\ref{sec:related-work}
places our work in the context of related work in more detail, and
\sect\ref{sec:Conclusions} concludes.

\section{Background}

\label{sec:background}In this paper we focus on factored Dec-POMDPs~\cite{Oliehoek08AAMAS},
which are Dec-POMDPs where the transition and observation models can
be represented compactly as a \emph{two-stage dynamic Bayesian network}~(2DBN)~\cite{Boutilier99JAIR}:\begin{definition}
A \emph{factored Dec-POMDP} is a tuple $\mathcal{M}=\langle\agentS,\jaS,\joS,\sfacS,\Tfunc,\Ofunc,\RS,\bO\rangle,$
where: 
\begin{itemize}
\item $\agentS=\{\agentI1,\dots,\agentI\nrA\}$ is the set of agents.
\item $\jaS=\bigotimes_{i\in\agentS}\aAS i$ is the set of joint actions
$\ja$.
\item $\joS=\bigotimes_{i\in\agentS}\oAS i$ is the set of joint observations
$\jo$.
\item $\sfacS=\left\{ \sfacI1,\dots,\sfacI\nrSF\right\} $ is a set of state
variables, or \emph{factors}, that take values $\sfacvIS k$ and thus
span the set of states $\sS=\bigotimes_{\sfacI k\in\sfacS}\sfacvIS k.$
\item $\Tfunc(\s'|\s,\ja)$ is the transition model which is specified by
a set of conditional probability tables (CPTs), one for each factor.
\item $\Ofunc(\jo|\ja,s')$ is the observation model, specified by a CPT
per agent.
\item $\RS=\left\{ \RI1,\dots,\RI\nrR\right\} $ is a set of \emph{local}
reward functions.
\item $\bO$ is the (factored) initial state distribution.
\end{itemize}
\end{definition}

Each local reward function $\RI l$ has a \emph{state factor scope}
$\sfacS(l)\subseteq\sfacS$ and \emph{agent scope} $\agentS(l)\subseteq\agentS$
over which is it is defined: $\RI l(\sfacGT{\sfacS(l)}{},\jaG{\agentS(l)},\sfacGT{\sfacS(l)}{\prime})\in\reals$.
These local reward functions form the global immediate reward function
via addition. We slightly abuse notation and overload $l$ to denote
both an index into the set of reward functions, as well as the corresponding
scopes: 
\[
\R(s,\ja,\s')\defas\sum_{l\in\RS}\RI l(\sfacvG l,\jaG l,\sfacvG l').
\]

Every Dec-POMDP can be converted to a factored Dec-POMDP, but the
additional structure that a factored model specifies is most useful
when the problem is \emph{weakly coupled}, meaning that there is sufficient
conditional independence in the 2DBN and that the scopes of the reward
functions are small. 

\begin{figure}
\begin{centering}
\includegraphics[width=0.8\columnwidth]{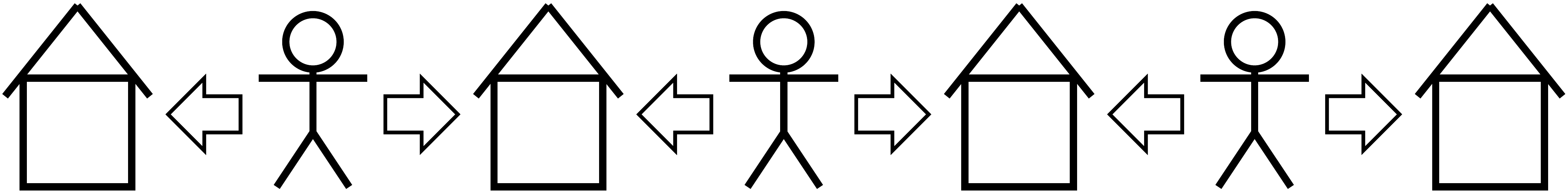}
\par\end{centering}

\protect\caption{The \<FFGfull> problem.}

\label{fig:FFG-problem}
\end{figure}

For instance, \fig \ref{fig:FFG-problem} shows the \<FFGfull> (\<FFG>)
problem \cite{Oliehoek13AAMAS}, which we adopt as a running example.
This problem defines a set of $\nrA+1$ houses, each with a particular
`fire level' indicating if the house is burning and with what intensity.
Each agent can fight fire at the house to its left or right, making
observations of flames (or no flames) at the visited house. Each house
has a local reward function associated with it, which depends on the
next-stage fire-level,%
\footnote{FFG has rewards of form $\RI l(\sfacvG l')$, but we support $\RI l(\sfacvG l,\jaG l,\sfacvG l')$
in general. %
} as illustrated in \fig\ref{fig:sub-problem}(left) which shows the
2DBN for a 4-agent instantiation of \<FFG>. The figure shows that
the connections are \emph{local} but there is no \emph{transition
independence}~\cite{Becker03AAMAS} or \emph{value factorization}
\cite{Kumar11IJCAI,Witwicki11PhD}: all houses and agents are connected
such that, over time, actions of each agent can influence the entire
system. While \<FFG> is a stylized example, such locally-connected
systems can be found in applications as traffic control \cite{Wu13IJCAI}
or communication networks~\cite{Ooi96,Hansen04AAAI,Mahajan14AOR}.

 This paper focuses on problems with a finite horizon~$\hor$ such
that $\ts=0,\dots,\hor-1$. A policy~$\polA i$ for an agent~$i$
specifies an action for each observation history $\oHistAT i{\ts}=(\oAT i1,\dots,\oAT i{\ts}).$
The task of planning for a factored Dec-POMDP entails finding a joint
policy $\jpol=\langle\polA 1,\dots,\polA{\nrA}\rangle$ with maximum
\emph{value}, i.e., expected sum of rewards:
\[
\V(\jpol)\defas\E\bigg[\sum_{\ts=0}^{\hor-1}\R(\s,\ja,\s')\mid\bO,\jpol\bigg].
\]
Such an optimal joint policy is denoted $\jpol^{*}$.%
\footnote{We omit the `{*}' on values; all values are assumed to be optimal
with respect to their given arguments.%
}

In recent years, a number of methods have been proposed to find approximate
solutions for factored Dec-POMDPs with many agents~\cite{Pajarinen11IJCAI,Kumar11IJCAI,Velagapudi11AAMAS,Oliehoek13AAMAS,Wu13IJCAI}
but none of these methods are able to give guarantees with respect
to the solution quality (i.e., they are \emph{heuristic }methods),
leaving the user unable to confidently gauge how well these methods
perform on their problems. This is a principled problem; even finding
an $\epsilon$-approximate solution is NEXP-complete~\cite{Rabinovich03AAMAS},
which implies that general and efficient approximation schemes are
unlikely to be found. In this paper, we propose a way forward by trying
to find instance-specific upper bounds in order to provide information
about the solution quality offered by heuristic methods.

\section{Sub-Problems and Influences}

\label{sec:sub-problems}
\begin{figure}
\begin{centering}
\begin{minipage}[t][60mm][c]{0.5\columnwidth}%
\begin{center}
\includegraphics[scale=0.62]{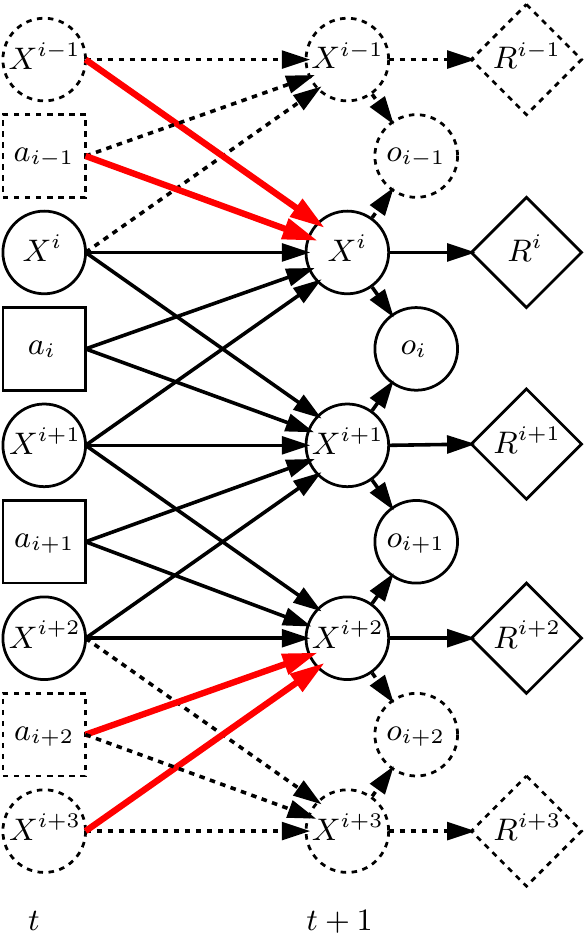}
\par\end{center}%
\end{minipage}%
\begin{minipage}[t][54mm][c]{0.5\columnwidth}%
\begin{center}
\includegraphics[scale=0.62]{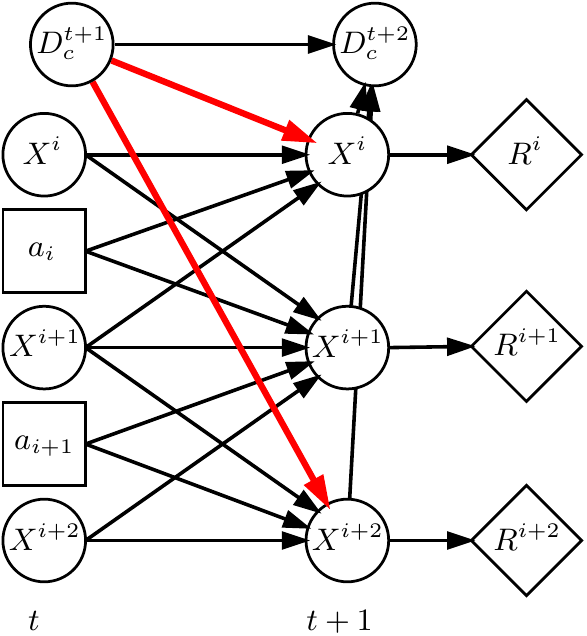}
\par\end{center}%
\end{minipage}
\par\end{centering}

\begin{centering}

\par\end{centering}

\protect\caption{Left: A 2-agent sub-problem within 4-agent \<FFG>. Right: the corresponding
IASP.}

\label{fig:sub-problem}\label{fig:IASP}
\end{figure}

The overall approach that we take is to divide the problem into sub-problems
(defined here), compute overestimations of the achievable value for
each of these sub-problems (discussed in \sect\ref{sec:local-upper-bounds})
and combine those into a global upper bound (\sect\ref{sec:global-upper-bounds}).

\subsection{Sub-Problems (SPs)}

The notion of a sub-problem generalizes the concept of a \emph{local-form
model (LFM) }\cite{Oliehoek12AAAI_IBA} to multiple agents and reward
components. We give a relatively concise description of this formalization,
for more details, please see~\cite{Oliehoek12AAAI_IBA}.

\begin{definition} A \emph{sub-problem (SP)} $\mathcal{M}_{\co}$\emph{
}of a factored Dec-POMDP $\mathcal{M}$ is a tuple $\mathcal{M}_{\co}=\langle\mathcal{M},\agentS',\sfacS',\RS'\rangle$,
where $\agentS'\subset\agentS,\sfacS'\subset\sfacS,\RS'\subset\RS$
denote subsets of agents, state factors and local reward functions.
\end{definition}

An SP inherits many features from $\mathcal{M}$: we can define local\emph{
}states $\mfC{\co}\in\bigotimes_{\sfac\in\sfacS'}$ and the subsets
$\agentS',\sfacS',\RS'$ induce local joint actions $\jaGS{\co}=\bigotimes_{i\in\agentS'}\aAS i$,
observations $\joGS{\co}=\bigotimes_{i\in\agentS'}\oAS i$, and rewards
\begin{equation}
\RC{\co}(\mfC{\co},\jaC{\co},\mfC{\co}')\defas\sum_{l\in\RS'}\RI l(\sfacvG l,\jaG l,\sfacvG l').\label{eq:R_c__local-reward}
\end{equation}
However, this is generally not enough to end up with a fully specified,
but smaller, factored Dec-POMDP. This is illustrated in \fig\ref{fig:sub-problem}(left),
which shows the 2DBN for a sub-problem of \<FFG> involving two agents
and three houses (dependence of observations $\oA i$ on actions $\aA i$
is not displayed). The figure shows that state factors $\sfac\in\sfacS'$ (in this
case $\sfacI{i}$ and $\sfacI{i+2}$) can be the target of arrows
pointing into the sub-problem from the non-modeled (dashed) part.
We refer to such state factors as \emph{non-locally affected factors
(NLAFs) }and denote them $\mfnCI{\co}k$, where $c$ indexes the SP
and $k$ indexes the factor. The other state factors in $\sfacS'$
are referred to as \emph{only-locally affected factors (OLAFs) }$\mflCI{\co}k$.
The figure clearly shows that the transition probabilities are not
well-defined since the NLAFs depend on the sources of the highlighted
\emph{influence links. }We refer to these sources as \emph{influence
sources} \emph{$\ifsCT{\co}{\ts+1}=\left\langle \nmfCT u{\ts},\jaGT u{\ts}\right\rangle $
}(in this case $\nmfCT u{\ts}=\left\langle \sfacI{i-1},\sfacI{i+3}\right\rangle $
and $\jaGT u{\ts}=\left\langle \aAT{i-1}{\ts},\aAT{i+2}{\ts}\right\rangle $).
This means that an SP $\co$ has an \emph{underspecified transition
model: $\Tfunc_{\co}(\mfCT{\co}{\ts+1}|\mfCT{\co}{\ts},\jaGT{\co}{\ts},\ifsCT{\co}{\ts+1})$.}

\subsection{Structural Assumptions}

In the most general form, the observation and reward model could also
be underspecified. In order to simplify the exposition, we make two
assumptions on the structure of an SP:
\begin{enumerate}
\item For all included agents $i\in\agentS'$, the state factors that can
influence its observations (i.e., ancestors of $\oA i$ in the 2DBN)
are included in $\mathcal{M}_{\co}$.
\item For all included reward components $\RI l\in\RS'$, the state factors
and actions that influence $\RI l$ are included in $\mathcal{M}_{\co}$.
\end{enumerate}
That is, we assume that SPs exhibit generalized forms of \emph{observation
independence},
\[
\Ofunc_{\co}(\joG\co|\jaC{\co},\mfC{\co}')\defas\Pr(\joG\co|\jaC{\co},\mfC{\co}')=\Pr(\joG\co|\ja,\s'),
\]
and \emph{reward independence }(cf.\ \eqref{eq:R_c__local-reward}).
These are more general notions of observation and reward independence
than used in previous work on TOI-Dec-MDPs~\cite{Becker03AAMAS}
and ND-POMDPs~\cite{Nair05AAAI}, since we allow overlap on state
factors that can be influenced by the agents themselves.%
\footnote{Previous work only allowed `external' or `unaffectable' state
factors to affect the observations or rewards of multiple components.%
}

Crucially, however, we do not assume any form of transition independence
(for instance, the sets $\sfacS'$ of SPs can overlap), nor do we
assume any of the transition-decoupling (i.e., TD-POMDP~\cite{Witwicki10ICAPS})
restrictions. That is, we \emph{neither }restrict which node types
can affect `private' nodes; nor do we disallow concurrent interaction
effects on `mutually modeled' nodes. 

This means that assumptions 1 and 2 (above) that we do make are without
loss of generality: it is possible to make any Dec-POMDP problem satisfy
them by introducing additional (dummy) state factors.%
\footnote{In contrast, TOI-Dec-MDPs and ND-POMDPs impose \emph{both }transition
and observation independence, thereby restricting consideration to
a proper subclass of those considered here. %
}

\subsection{Influence-Augmented SPs}

An LFM can be transformed into a so-called \emph{influence-augmented
local model}, which captures the influence of the policies and parts
of the environment that are not modeled in the local model~\cite{Oliehoek12AAAI_IBA}.
Here we extend this approach to SPs, thus leading to \emph{influence-augmented
sub-problems (IASPs). }

Intuitively, the construction of an IASP consists of two steps: 1)
capturing the influence of the non-modeled parts of the problem (given
$\jpolG{\excl{\co}}$ the policies of non-modeled agents) in an \emph{incoming
influence point}~$\ifpiC{\co}(\jpolG{-c})$, and 2) using this $\ifpiC{\co}$
to create a model with a transformed transition model $\Tfunc_{\ifpiC{\co}}$
and no further dependence on the external problem.

Step (1) can be done as follows: an \emph{incoming influence point
}can be specified as an \emph{incoming influence} $\ifpiCT{\co}{\ts}$
for each stage: $\ifpiC{\co}=\left(\ifpiCT{\co}1,\dots,\ifpiCT{\co}h\right)$.
Each such $\ifpiCT{\co}{\ts+1}$ corresponds to the influence that
the SP experiences at stage $\ts+1$, and thus specifies the conditional
probability distribution of the influence sources\emph{ $\ifsCT{\co}{\ts+1}=\left\langle \nmfCT u{\ts},\jaGT u{\ts}\right\rangle $.
}That is, assuming that the influencing agents use deterministic policies~$\jpolG u=\langle\polA i\rangle_{i\in u}$
that map observation histories to actions, $\ifpiCT{\co}{\ts+1}$
is the conditional probability distribution given by 
\begin{multline*}
\iffunc(\ifsCT{\co}{\ts+1}\,|\,\dsetCT{\co}{\ts+1})=\\
\sum_{\oHistGT u\ts}\IndFo{\jaGT u{\ts}=\jpolG u(\oHistGT u\ts)}\Pr(\nmfCT u{\ts},\oHistGT u\ts|\dsetCT{\co}{\ts+1},\bO,\jpolG{\excl{\co}}),
\end{multline*}
where $\IndFo{\cdot}$ is the Kronecker Delta function, and $\dsetCT{\co}{\ts+1}$
the \emph{d-separating set }for $\ifpiCT{\co}{\ts+1}$: the history
of a subset of all the modeled variables that d-separates the modeled
variables from the non-modeled ones.%
\footnote{$\dsetCT{\co}{\ts+1}$ is defined such that $\Pr(\nmfCT u{\ts},\oHistGT u\ts|\dsetCT{\co}{\ts+1},\bO,\jpolG{\excl{\co}},\aoHistGT{\co}{\ts})$
$=$$\Pr(\nmfCT u{\ts},\oHistGT u\ts|\dsetCT{\co}{\ts+1},\bO,\jpolG{\excl{\co}})$,
see \cite{Oliehoek12AAAI_IBA} for details.%
}

Step (2) involves replacing the CPTs for all the NLAFs by the CPTs
\emph{induced }by $\ifpiC{\co}$. \begin{definition}Let $\mfnCIT{\co}k{\ts+1}$
be an NLAF (with index~$k$), and $\ifsCT{\co}{\ts+1}$ (the instantiation
of) the corresponding influence sources. Given the influence $\ifpiCT{\co}{\ts+1}(\jpolG{\excl{\co}})$,
and its d-separating set $\dsetAT i{\ts+1}$, we define the \emph{induced
CPT }for\emph{ }$\mfnCIT{\co}k{\ts+1}$ as the CPT that specifies
probabilities: 
\begin{multline}
p_{\ifpiCT{\co}{\ts+1}}(\mfnCIT{\co}k{\ts+1}|\mfCT{\co}{\ts},\dsetCT{\co}{\ts+1},\jaGT{\co}{\ts})=\\
\sum_{\ifsCT{\co}{\ts+1}=\left\langle \nmfAT u{\ts},\jaGT u{\ts}\right\rangle }\Pr(\mfnCIT{\co}k{\ts+1}|\mfCT{\co}{\ts},\jaGT{\co}{\ts},\ifsCT{\co}{\ts+1})\iffunc(\ifsCT{\co}{\ts+1}|\dsetCT{\co}{\ts+1}).\label{eq:induced-CPT}
\end{multline}
\end{definition}

Finally, we can define the IASP.

\begin{definition} An \emph{influence-augmented SP (IASP)} $\mathcal{M}_{\co}^{IA}$
$=\langle\mathcal{M}_{\co},\ifpiC{\co}\rangle$ for an SP $\mathcal{M}_{\co}=\langle\mathcal{M},\agentS',\sfacS',\RS'\rangle$
is a factored Dec-POMDP with the following components: 
\begin{itemize}
\item The agents (implying the actions and observations) from the respective
subproblem participate: $\bar{\agentS}=\agentS'$.
\item The set of state factors is $\bar{\sfacS}=\sfacS'\cup\{\dsetCT{\co}{}\}$
such that states $\IASPstateT{\ts}=\langle\mfCT{\co}{\ts},\dsetCT{\co}{\ts+1}\rangle$
specify a local state of the SP, as well as the d-separating set $\dsetAT i{\ts+1}$
for the next-stage influences. 
\item Transitions are specified as follows: For all OLAFs~$\mfCI{\co}k$
we take the CPTs from the factored Dec-POMDP~$\mathcal{M}$, but
for all NLAFs we take the induced CPTs, leading to an influence-augmented
transition model which is the product of CPTs of OLAFs and NLAFS:
\begin{multline}
\bar{\Tfunc}_{\ifpiC{\co}}(\mfCT{\co}{\ts+1}|\langle\mfCT{\co}{\ts},\dsetCT{\co}{\ts+1}\rangle,\jaGT{\co}{\ts})=\Pr(\mflCT{\co}{\ts+1}|\mfCT{\co}{\ts},\jaGT{\co}{\ts})\\
\sum_{\ifsCT{\co}{\ts+1}=\left\langle \nmfAT u{\ts},\jaGT u{\ts}\right\rangle }\Pr(\mfnCT{\co}{\ts+1}|\mfCT{\co}{\ts},\jaGT{\co}{\ts},\ifsCT{\co}{\ts+1})\iffunc(\ifsCT{\co}{\ts+1}|\dsetCT{\co}{\ts+1}).\label{eq:induced_T}
\end{multline}
(Note that $\mfCT{\co}{\ts},\jaGT{\co}{\ts},\mfCT{\co}{\ts+1}$ and
$\dsetCT{\co}{\ts+1}$ together uniquely specify $\dsetCT{\co}{\ts+2}$). 
\item The observation model $\bar{\Ofunc}$ follows directly from $\Ofunc$
(from \emph{$\mathcal{M}$}). 
\item The reward is identical to that of the SP: $\bar{\RS}=\RS'$. 
\end{itemize}
\end{definition}\fig\ref{fig:IASP}(right) illustrates the IASP
for FFG. It shows how the d-separating set acts as a parent for all
NLAFs, thus replacing the dependence on the external part of the problem.

We write $\argsC{\V}{\co}(\jpol)$ for the value that would be realized
for the reward components modeled in sub-problem $\co$, under a given
joint policy $\jpol$: 
\[
\argsC{\V}{\co}(\jpol)\defas\E\left[\sum_{\ts=0}^{h-1}\RCT{\co}{\ts}(\s,\ja,\s')\mid\bO,\jpol\right].
\]
 As one can derive, given the policies of other agents $\jpolG{\excl{\co}}$,
$V_{\co}(\ifpiC{\co}(\jpolG{\excl{\co}}))$, the value of the optimal
solution of an IASP constructed for the influence corresponding to
$\jpolG{\excl{\co}}$, is equal to the best-response value: 
\begin{equation}
V_{\co}^{BR}(\jpolG{\excl{\co}})\defas\max_{\jpolG{\co}}\argsC{\V}{\co}(\jpolG{\co},\jpolG{\excl{\co}})=V_{\co}(\ifpiC{\co}(\jpolG{\excl{\co}})).
\end{equation}
This extends the result in \cite{Oliehoek12AAAI_IBA} to multiagent
SPs.

\section{Local Upper Bounds }

\label{sec:local-upper-bounds}In this section we present our main
technical contribution: the machinery to compute \emph{influence-optimistic
upper bounds (IO-UBs) }for the value of sub-problems. In order to
properly define this class of upper bound, we first define the \emph{locally-optimal
value: }\begin{definition}The \emph{locally-optimal value }for an
SP~$\co$, 
\begin{equation}
V_{\co}^{LO}\defas\max_{\jpolG{\excl{\co}}}V_{\co}^{BR}(\jpolG{\excl{\co}})=\max_{\jpolG{\excl{\co}}}V_{\co}(\ifpiC{\co}(\jpolG{\excl{\co}})),\label{eq:V^LO}
\end{equation}
is\emph{ }the local value (considering only the rewards $\RC{\co}$)
that can be achieved when all agents use a policy selected to optimize
this local value. We will denote the maximizing argument by $\jpolG{\excl{\co}}^{LO}$.
\end{definition} 

Note that $V_{\co}^{LO}\geq V_{\co}(\jpol^{*})$---the value for the
rewards~$\RC{\co}$ under the optimal joint policy $\jpol^{*}$---since
$\jpol^{*}$ optimizes the sum of all local reward functions: it might
be optimal to sacrifice some reward $\RC{\co}$ if it is made up by
higher rewards outside of the sub-problem.

$V_{\co}^{LO}$ expresses the maximal value achievable under a \emph{feasible}
incoming influence point; i.e., it is optimistic about the influence,
but maintains that the influence is feasible. Computing this value
can be difficult, since computing influences and subsequently constructing
and optimally solving an IASP can be very expensive in general. However,
it turns out computing upper bounds to $V_{\co}^{LO}$ can be done
more efficiently, as discussed in \sect\ref{sec:complexity-analysis}.
\textbf{}

The IO-UBs that we propose in the remainder of this section upper-bound
$V_{\co}^{LO}$ by relaxing the requirement of the incoming influence
being feasible, thus allowing for more efficient computation. We present
three approaches that each overestimate the value by being optimistic
with respect to the assumed influence, but that differ in the additional
assumptions that they make.

\subsection{A Q-MMDP Approach}

\label{sec:MMDP-based}

The first approach we consider is called \emph{influence-optimistic
Q-MMDP (IO-Q-MMDP)}. Like all the heuristics we introduce, it assumes
that the considered SP will receive the most optimistic (possibly
infeasible) influence. In addition, it assumes that the SP is fully
observable such that it reduces to a local multiagent MDP (MMDP)~\cite{Boutilier96AAAI}.
In other words, this approach resembles Q-MMDP~\cite{Szer05UAI_MAA,Oliehoek08JAIR},
but is applied to an SP, and performs an influence-optimistic estimation
of value.%
\footnote{What we have termed ``Q-MMDP'' has been referred to in past work
as ``Q-MDP''; we add the extra M emphasize the presence of multiple
agents.%
} IO-Q-MMDP makes, in addition to influence optimism, another overestimation
due to its assumption of full observability. While this negatively
affects the tightness of the upper bound, it has as the advantage
that its computational complexity is relatively low.

Formally, we can describe IO-Q-MMDP as follows. In the first phase,
we apply dynamic programming to compute the action-values for all
local states:{\small
\begin{multline}
\negmedspace\negmedspace\negmedspace\negmedspace\Q(\mfCT{\co}{\ts},\jaGT{\co}{\ts})=\max_{\ifsCT{\co}{\ts+1}}\sum_{\mfCT{\co}{\ts+1}}\Pr(\mflCT{\co}{\ts+1}|\mfCT{\co}{\ts},\jaGT{\co}{\ts})\Pr(\mfnCT{\co}{\ts+1}|\mfCT{\co}{\ts},\jaGT{\co}{\ts},\ifsCT{\co}{\ts+1})\\
\Big[\RC{\co}(\mfCT{\co}{\ts},\jaGT{\co}{\ts},\mfCT{\co}{\ts+1})+\max_{\jaGT{\co}{\ts+1}}\Q(\mfCT{\co}{\ts+1},\jaGT{\co}{\ts+1})\Big].\label{eq:IO-Q-QMMDP}
\end{multline}
}Comparing this equation to \eqref{eq:induced_T}, it is clear that
this equation is optimistic with respect to the influence: it selects
the sources $\ifsCT{\co}{\ts+1}$ in order to select the most beneficial
transition probabilities. In the second phase, we use these values
to compute an upper bound:
\[
\argsC{\hat{\V}}{\co}^{M}\defas\max_{\jaG{\co}}\sum_{\mfCT{\co}{}}\bO(\mfCT{\co}{})Q(\mfCT{\co}{},\jaG{\co}).
\]

This procedure is guaranteed to yield an upper bound to the locally-optimal
value for the SP. 

\begin{theorem}\label{thm:IO-Q-MMDP-is-UB}IO-Q-MMDP yields an upper
bound to the locally-optimal value: $V_{\co}^{LO}\leq\argsC{\hat{\V}}{\co}^{M}$.\end{theorem}\begin{proof}An
inductive argument easily establishes that, due to the maximization
it performs, \eqref{eq:IO-Q-QMMDP} is at least as great as the Q-MMDP
value (for all~$\dsetCT{\co}{\ts+1}$) of \emph{any} feasible influence,
given by:
\begin{multline}
\Q_{\co}^{\mathit{MMDP}}(\langle\mfCT{\co}{\ts},\dsetCT{\co}{\ts+1}\rangle,\jaGT{\co}{\ts})=\sum_{\mfCT{\co}{\ts+1}}\bar{T}_{\ifpiC{\co}}(\mfCT{\co}{\ts+1}|\langle\mfCT{\co}{\ts},\dsetCT{\co}{\ts+1}\rangle,\jaGT{\co}{\ts})\\
\Big[\RC{\co}(\mfCT{\co}{\ts},\jaGT{\co}{\ts},\mfCT{\co}{\ts+1})+\max_{\jaGT{\co}{\ts+1}}\Q(\langle\mfCT{\co}{\ts+1},\dsetCT{\co}{\ts+2}\rangle,\jaGT{\co}{\ts+1})\Big].\label{eq:Q-QMMDP}
\end{multline}
Therefore the value computed in \eqref{eq:IO-Q-QMMDP} is at least
as great as the Q-MMDP value~\eqref{eq:Q-QMMDP} induced by $\jpolG{\excl{\co}}^{LO}$
(the maximizing argument of \eqref{eq:V^LO}), for all $\mfCT{\co}{\ts},\dsetCT{\co}{\ts+1},\jaGT{\co}{\ts}$.
This directly implies
\[
\argsC{\hat{\V}}{\co}^{M}\geq V_{\co}^{\mathit{MMDP}}(\ifpiC{\co}(\jpolG{\excl{\co}}^{LO})),
\]
 Moreover, it is well known that, for any Dec-POMDP, the Q-MMDP value
is an upper bound to its value \cite{Szer05UAI_MAA}, such that 
\[
V_{\co}^{\mathit{MMDP}}(\ifpiC{\co}(\jpolG{\excl{\co}}^{LO}))\geq V_{\co}(\ifpiC{\co}(\jpolG{\excl{\co}}^{LO})).
\]
We can conclude that $\argsC{\hat{\V}}{\co}^{M}$ is an upper bound
to the Dec-POMDP value of the IASP induced by $\jpolG{\excl{\co}}^{LO}$:
\[
\argsC{\hat{\V}}{\co}^{M}\geq V_{\co}(\ifpiC{\co}(\jpolG{\excl{\co}}^{LO}))=\max_{\jpolG{\excl{\co}}}V_{\co}(\ifpiC{\co}(\jpolG{\excl{\co}}))=V_{\co}^{LO},
\]
with the identities given by \eqref{eq:V^LO}, thus proving the theorem.\end{proof}

The upshot of \eqref{eq:IO-Q-QMMDP} is that there are no dependencies
on d-separating sets and incoming influences anymore: the IO assumption
effectively eliminates these dependencies. As a result, there is no
need to actually construct the IASPs (that potentially have a very
large-state-space) if all we are interested in is an upper bound.

\subsection{A Q-MPOMDP Approach}

\label{sec:QMPOMDP-based}\textbf{}The IO-Q-MMDP approach of the
previous section introduces overestimations through influence-optimism
as well as by assuming full observability. Here we tighten the upper
bound by weakening the second assumption. In particular, we propose
an upper bound based on the underlying \emph{multiagent POMDP~(MPOMDP). }

An MPOMDP~\cite{Messias11NIPS24,Amato13MSDM} is partially observable,
but assumes that the agents can freely communicate their observations,
such that the problem reduces to a special type of centralized model
in which the decision maker (representing the entire team of agents)
takes joint actions, and receives joint observations. As a result,
the optimal value for an MPOMDP is analogous to that of a POMDP:
\begin{equation}
Q(\jbT\ts,\jaT{\ts})=R(\jbT\ts,\jaT{\ts})+\sum_{\joT{\ts+1}}\Pr(\joT{\ts+1}|\jbT\ts,\jaT{\ts})V(\jbT{\ts+1}),\label{eq:Q-POMDP}
\end{equation}
where $\jbT{\ts+1}$ is the joint belief resulting from performing
Bayesian updating of $\jbT{\ts}$ given $\jaT{\ts}$ and $\joT{\ts+1}$. 

Using the value function of the MPOMDP solution as a heuristic (i.e.,
an upper bound) for the value function of a Dec-POMDP is a technique
referred to as Q-MPOMDP~\cite{Roth05AAMAS,Oliehoek08JAIR}\emph{.}
Here we combine this approach with optimistic assumptions on the influences,
leading to \emph{influence-optimistic Q-MPOMDP (IO-Q-MPOMDP).}

In case that the influence on an SP is fully specified, \eqref{eq:Q-POMDP}
can be readily applied to the IASP. However, we want to deal with
the case where this influence is not specified. The basic, conceptually
simple, idea is to move from the \emph{influence-optimistic }MMDP-based
upper bounding scheme from in \sect\ref{sec:MMDP-based} to one based
on MPOMDPs. However, it presents a technical difficulty, since it
is not directly obvious how to extend \eqref{eq:IO-Q-QMMDP} to deal
with partial observability. In particular, in the MPOMDP case as given
by \eqref{eq:Q-POMDP}, the state $\mfCT{\co}{\ts}$ is replaced by
a belief over such local states and the influence sources $\ifsCT{\co}{\ts+1}$
affect the value by both manipulating the transition and observation
probabilities, as well as the resulting beliefs.

To overcome these difficulties, we propose a formulation that is not directly based on \eqref{eq:Q-POMDP}, but that too makes
use of `back-projected value vectors'. That is, it is possible to rewrite the optimal MPOMDP value function
as:%
\footnote{In this section and the next, we will restrict ourselves to rewards
of the form $\R(\s,\ja)$ to reduce the notational burden, but the
presented formulas can be extended to deal with $\R(\s,\ja,\s')$
formulations in a straightforward way.%
}
\begin{equation}
Q(\jbT\ts,\jaT{\ts})=\inprod{\jbT\ts}{r_{a}}+\gamma\sum_{\joT{\ts+1}}\max_{\nu^{ao}\in\mathcal{V}^{ao}}\inprod{\jbT\ts}{\nu^{ao}},\label{eq:Q-POMDP--back-proj-based}
\end{equation}
where $\inprod{}{}$ denotes inner product and where $\nu^{ao}\in\mathcal{V}^{ao}$
are the back-projections of value vectors $\nu\in\mathcal{V}^{\ts+1}$:
\begin{equation}
\nu^{ao}(\sT{\ts})\defas\sum_{\sT{\ts+1}}O(\joT{\ts+1}|\jaT{\ts},\sT{\ts+1})T(\sT{\ts+1}|\sT{\ts}\jaT{\ts})\nu(\sT{\ts+1}).\negthickspace\label{eq:BackProj_POMDP}
\end{equation}
(Please see, e.g., \cite{Spaan12RLBook,Shani13JAAMAS} for more details.)

The key insight that enables carrying influence-opti-mism to the MPOMDP
case is that this back-projected form \eqref{eq:BackProj_POMDP} does
allow us to take the maximum with respect to unspecified influences.
That is, we define the \emph{influence-optimistic back-projection
}as:
\begin{multline}
\nu_{IO}^{ao}(\mfCT{\co}{\ts})\defas\max_{\ifsCT{\co}{\ts+1}}\sum_{\mfCT{\co}{\ts+1}}O(\joGT{\co}{\ts+1}|\jaG{\co},\mfCT{\co}{\ts+1})\\
\Pr(\mfnCT{\co}{\ts+1}|\mfCT{\co}{\ts},\jaG{\co},\ifsCT{\co}{\ts+1})\Pr(\mflCT{\co}{\ts+1}|\mfCT{\co}{\ts},\jaG{\co})\nu_{IO}(\mfCT{\co}{\ts+1}).\label{eq:IO-BackProj_POMDP}
\end{multline}

Since this equation does not depend in any way on the d-separating
sets and influence, we can completely avoid generating large IASPs.
As for implementation, many POMDP solution methods~\cite{Cassandra97UAI,Kaelbling98AI}
are based on such back-projections and therefore can be easily modified;
all that is required is to substitute these the back-projections by
their modified form \eqref{eq:IO-BackProj_POMDP}. When combined with
an \emph{exact }POMDP solver, such influence-optimistic back-ups will
lead to an upper bound $\argsC{\hat{\V}}{\co}^{P}$, to which we refer
as \emph{IO-Q-MPOMDP}, on the locally-optimal value. 

To formally prove this claim, we will need to discriminate a few different
types of value, and associated constructs. Let us define:
\begin{itemize}
\item $\IASPbelIT{\ifpiC{\co}}{\ts}(\langle\mfCT{\co}{\ts},\dsetCT{\co}{\ts+1}\rangle)$,
an MPOMDP belief for the IASP induced by an arbitrary influence $\ifpiC{\co}$,
\item $\V_{\ifpiC{\co}}^{P}$, the optimal value function when the IASP
is solved as an MPOMDP, such that $\V_{\ifpiC{\co}}^{P}(\IASPbelIT{\ifpiC{\co}}{\ts})$
is the value of $\IASPbelIT{\ifpiC{\co}}{\ts}$ and $\mathit{\V}_{\co}^{MPOMDP}(\ifpiC{\co})\defas\V_{\ifpiC{\co}}^{P}(\IASPbelIT{\ifpiC{\co}}0)$.
$\V_{\ifpiC{\co}}^{P}$ is represented using vectors $\nu\in\mathcal{V}$.
\item $\IASPbelIT{IO}{\ts}(\mfCT{\co}{\ts})$, an arbitrary distribution
over $\mfCT{\co}{\ts}$ that can be thought of as the MPOMDP belief
for the `influence optimistic SP'.%
\footnote{For instance, in the case of FFG from \fig\ref{fig:sub-problem},
we can imagine an SP that encodes optimistic assumptions by assuming
that the neighboring agents always will fight fire at houses $i$
and $i+2$. Even though it may not be possible to define such a model
for all problems---the optimistic influence could depend on the local
(belief) state in intricate ways---this gives some interpretation
to $\IASPbelIT{IO}{\ts}(\mfCT{\co}{\ts}).$ Additionally, we exploit
the fact that construction of an optimistic SP model is possible for
the considered domains in \sect\ref{sec:evaluation}.%
}
\item $\V_{IO,\co}^{P}$, the value function computed by an (exact) influence-optimistic
MPOMDP method, that assigns a value $\V_{IO,\co}^{P}(\IASPbelIT{IO}{\ts})$
to any $\IASPbelIT{IO}{\ts}$. $\V_{IO,\co}^{P}$ is represented using
vectors $\nu_{IO}\in\mathcal{V}_{IO}$. The IO-Q-MPOMDP upper bound
is defined by plugging in the true initial state distribution~$\bO$,
restricted to factors in $\co$: $\argsC{\hat{\V}}{\co}^{P}\defas\V_{\co}^{P}(\IASPbelIT{IO}0)$. 
\end{itemize}
First we establish a relation between the different vectors representing
$\V_{\ifpiC{\co}}^{P}$ and $\V_{IO,\co}^{P}$.

\begin{lemma}\label{lem:maxD_v__leq__vIO}\textbf{L}et $\jpolG{\co}^{t:h-1}$
be a $(h-\ts)$-steps-to-go policy. Let $\nu\in\mathcal{V}$ and $\nu_{IO}\in\mathcal{V}_{IO}$
be the vectors induced by $\jpolG{\co}^{t:h-1}$ under regular MPOMDP
back-projections (for some $\ifpiC{\co}$), and under IO back-projections
respectively. Then 
\[
\forall_{\mfCT{\co}{\ts}}\qquad\max_{\dsetCT{\co}{\ts+1}}\nu(\langle\mfCT{\co}{\ts},\dsetCT{\co}{\ts+1}\rangle)\leq\nu_{IO}(\mfCT{\co}{\ts}).
\]
 \textbf{}\begin{proof}The proof is listed in \app\ref{sec:app:proofs}.\end{proof}\end{lemma}

This lemma provides a strong result on the relation of values computed
under regular MPOMDP backups versus influence-optimistic ones. It
allows us to establish the following theorem:

\begin{theorem}\label{thm:QMMDP_leq_IOQMMDP}For an SP $\co$, for
all $\ifpiC{\co}$, 
\[
\forall_{\IASPbelIT{\ifpiC{\co}}{\ts}}\qquad\V_{\ifpiC{\co}}^{P}(\IASPbelIT{\ifpiC{\co}}{\ts})\leq\V_{IO,\co}^{P}(\IASPbelIT{IO}{\ts}),
\]
provided that $\IASPbelIT{IO}{\ts}$ concides with the marginals of
$\IASPbelIT{\ifpiC{\co}}{\ts}$: 
\[
\text{(A1)}\hspace{1em}\forall_{\mfCT{\co}{\ts}}\;\sum_{\dsetCT{\co}{\ts+1}}\IASPbelIT{\ifpiC{\co}}{\ts}(\langle\mfCT{\co}{\ts},\dsetCT{\co}{\ts+1}\rangle)=\IASPbelIT{IO}{\ts}(\mfCT{\co}{\ts}).
\]
\end{theorem}\begin{proof}  We start with the left hand side:
\begin{align*}
 & \hskip-3emV_{\ifpiC{\co}}^{P}(\IASPbelT{\ts})=\max_{\nu\in\mathcal{V}}\inprod{\IASPbelIT{\ifpiC{\co}}{\ts}}{\nu}\\
= & \max_{\nu\in\mathcal{V}}\sum_{\langle\mfCT{\co}{\ts},\dsetCT{\co}{\ts+1}\rangle}\IASPbelIT{\ifpiC{\co}}{\ts}(\langle\mfCT{\co}{\ts},\dsetCT{\co}{\ts+1}\rangle)\nu(\langle\mfCT{\co}{\ts},\dsetCT{\co}{\ts+1}\rangle)\\
\leq & \max_{\nu\in\mathcal{V}}\sum_{\mfCT{\co}{\ts}}\sum_{\dsetCT{\co}{\ts+1}}\IASPbelIT{\ifpiC{\co}}{\ts}(\langle\mfCT{\co}{\ts},\dsetCT{\co}{\ts+1}\rangle)\max_{\dsetCT{\co}{\ts+1}}\nu(\langle\mfCT{\co}{\ts},\dsetCT{\co}{\ts+1}\rangle)\\
\text{\{A1\}}= & \max_{\nu\in\mathcal{V}}\sum_{\mfCT{\co}{\ts}}\IASPbelIT{IO}{\ts}(\mfCT{\co}{\ts})\max_{\dsetCT{\co}{\ts+1}}\nu(\langle\mfCT{\co}{\ts},\dsetCT{\co}{\ts+1}\rangle)\\
 & \text{\{\lem\ref{lem:maxD_v__leq__vIO}\}}\\
\leq & \max_{\nu_{IO}\in\mathcal{V}_{IO}}\sum_{\mfCT{\co}{\ts}}\IASPbelIT{IO}{\ts}(\mfCT{\co}{\ts})\nu_{IO}(\mfCT{\co}{\ts})\\
= & \max_{\nu_{IO}\in\mathcal{V}_{IO}}\inprod{\IASPbelIT{IO}{\ts}}{\nu_{IO}}\\
= & \, V_{IO,\co}^{P}(\IASPbelIT{IO}{\ts}),
\end{align*}
thus proving the theorem.\end{proof}

\begin{corollary}\label{cor:IO-Q-MPOMDP-is-UB}IO-Q-MPOMDP yields
an upper bound to the locally-optimal value: $V_{\co}^{LO}\leq\argsC{\hat{\V}}{\co}^{P}$.\end{corollary}\begin{proof}The
initial beliefs are defined such that the above condition (A1) holds.
That is:
\[
\sum_{\dsetCT{\co}{\ts+1}}\IASPbelIT{\ifpiC{\co}}0(\langle\mfCT{\co}0,\dsetCT{\co}1=\emptyset\rangle)=\IASPbelIT{IO}0(\mfCT{\co}{\ts})=\bO(\mfCT{\co}{\ts}).
\]
Therefore, application of \thm\ref{thm:QMMDP_leq_IOQMMDP} to the
initial belief yields: 
\[
\forall_{\ifpiC{\co}}\;\argsC{\hat{\V}}{\co}^{P}\defas\V_{IO,\co}^{P}(\IASPbelIT{IO}0)\geq\V_{\ifpiC{\co}}^{P}(\IASPbelIT{\ifpiC{\co}}0)\defas\V_{\co}^{\mathit{MPOMDP}}(\ifpiC{\co})
\]
It is well-known that the MPOMDP value is an upper bound to the Dec-POMDP
value~\cite{Oliehoek08JAIR}, such that 
\[
V_{\co}^{\mathit{MPOMDP}}(\ifpiC{\co}(\jpolG{\excl{\co}}^{LO}))\geq V_{\co}(\ifpiC{\co}(\jpolG{\excl{\co}}^{LO})),
\]
and we can immediately conclude that
\[
\argsC{\hat{\V}}{\co}^{P}\geq V_{\co}(\ifpiC{\co}(\jpolG{\excl{\co}}^{LO}))=\max_{\jpolG{\excl{\co}}}V_{\co}(\ifpiC{\co}(\jpolG{\excl{\co}}))=V_{\co}^{LO},
\]
with the identities given by \eqref{eq:V^LO}, proving the result.\end{proof}

\subsection{A Dec-POMDP Approach}

\label{sec:UB:Dec-POMDP}

The previous approaches compute upper bounds by, apart from the IO
assumption, additionally making optimistic assumptions on observability
or communication capabilities.  Here we present a general method
for computing \emph{Dec-POMDP-based upper bounds} that, other than
the optimistic assumptions about neighboring SPs, make no additional
assumptions and thus provide the tightest bounds out of the three
that we propose. This approach builds on the recent insight \cite{MacDermed13NIPS26,Dibangoye13IJCAI,Oliehoek13JAIR}
that a Dec-POMDP can be converted to a special case of POMDP (for
an overview of this reduction, see \cite{Oliehoek14IASTR}); we can
thereby leverage the influen\-ce-optimistic back-projection \eqref{eq:IO-BackProj_POMDP}
to compute an IO-UB that we refer to as \emph{IO-Q-Dec-POMDP.}

As\textbf{ }in the previous two sub-sections, we will leverage optimism
with respect to an influence-augmented model that we will never need
to construct. In particular, as explained in \sect\ref{sec:sub-problems}
we can convert an SP $\mathcal{M}_{\co}$ to an IASP $\mathcal{M}_{\co}^{IA}$
given an influence $\ifpiC{\co}$. Since such an IASP is a Dec-POMDP,
we can convert it to a special case of a POMDP:

\begin{definition}{\sloppypar A \emph{plan-time influence-augmented
sub-problem, }$\mathcal{M}_{\co}^{\text{PT-IA}}$, is a tuple $\mathcal{M}_{\co}^{\text{PT-IA}}(\mathcal{M}_{\co},\ifpiC{\co})=\left\langle \PT{\sS},\PT{\aAS\agNONE},\PT T_{\ifpiC{\co}},\PT R,\PT{\oAS\agNONE},\PT O,\PT{\hor},\PT{\bO}\right\rangle $,
where: 
\begin{itemize}
\item $\PT{\sS}$ is the set of states $\PTIASPstateT{\ts}=\left\langle \IASPstateT{\ts},\oHistGT{\co}\ts\right\rangle =\left\langle \mfCT{\co}{\ts},\dsetCT{\co}{\ts+1},\oHistGT{\co}\ts\right\rangle $.
\item $\PT{\aAS\agNONE}$ is the set of actions, each $\PTaT{\ts}$ corresponds
to a local joint decision rule $\jdrGT{\co}{\ts}$ in the SP.
\item $\PT{\Tfunc}_{\ifpiC{\co}}(\PTIASPstateT{\ts+1}|\PTIASPstateT{\ts},\PTaT{\ts})$
is the transition function defined below.
\item $\PT R(\PTIASPstateT{\ts},\PTaT{\ts})=\RC{\co}(\mfCT{\co}{\ts},\jdrGT{\co}{\ts}(\oHistGT{\co}\ts))$.
\item $\PT{\oAS\agNONE}=\left\{ \mathit{NULL}\right\} $.
\item $\PT O$ specifies that observation $\mathit{NULL}$ is received with
probability 1 (irrespective of the state and action).
\item The horizon is not modified: $\PT{\hor}=\hor$.
\item $\PT{\bO}$ is the initial state distribution. Since there is only
one $\oHistT{0}$ (i.e., the empty joint observation history).
\end{itemize}
}

The transition function specifies:
\begin{multline*}
\PT{\Tfunc}_{\ifpiC{\co}}(\left\langle \mfCT{\co}{\ts+1},\dsetCT{\co}{\ts+2},\oHistGT{\co}{\ts+1}\right\rangle |\left\langle \mfCT{\co}{\ts},\dsetCT{\co}{\ts+1},\oHistGT{\co}\ts\right\rangle ,\jdrGT{\co}{\ts})\\
\defas\bar{\Tfunc}_{\ifpiC{\co}}(\mfCT{\co}{\ts+1}|\langle\mfCT{\co}{\ts},\dsetCT{\co}{\ts+1}\rangle,\jdrGT{\co}{\ts}(\oHistGT\co\ts))\bar{\Ofunc}(\joGT{\co}{\ts+1}|\jdrGT{\co}{\ts}(\oHistGT\co\ts),\mfCT{\co}{\ts+1})
\end{multline*}
if $\oHistGT\co{\ts+1}=\left(\oHistGT\co{\ts},\joGT{\co}{\ts+1}\right)$
and 0 otherwise. In this equation $\bar{\Tfunc}_{\ifpiC{\co}},\bar{\Ofunc}$
are given by the IASP (cf. \sect\ref{sec:sub-problems}).%
\footnote{Remember that $\dsetCT{\co}{\ts+2}$ is a function of the specified
quantities: $\dsetCT{\co}{\ts+2}=d(\mfCT{\co}{\ts},\dsetCT{\co}{\ts+1},\jdrGT{\co}{\ts}(\oHistGT\co\ts),\mfCT{\co}{\ts+1})$.%
}

\end{definition}

This reduction shows that it is possible to compute $V_{c}^{*}(\ifpiC{\co}(\jpolG{\excl{\co}}))$,
the optimal value for an SP \emph{given an influence point $\ifpiC{\co}$},
but the formulation is subject to the same computational burden as
solving a regular IASP: constructing it is complex due to the inference
that needs to be performed to compute $\ifpiC{\co}$, and subsequently
solving the IASP is complex due to the large number of augmented states
$\PTIASPstateT{\ts}=\left\langle \mfCT{\co}{\ts},\dsetCT{\co}{\ts+1},\oHistGT{\co}\ts\right\rangle $.

Fortunately, here too we can compute an upper bound to \emph{any}~feasible
incoming influence, and thus to $V_{c}^{LO}$, by using optimistic
backup operations with respect to an underspecified model, to which
we refer as simply \emph{plan-time SP:}

\begin{definition}\sloppypar We define the \emph{plan-time sub-problem
$\mathcal{M}_{\co}^{\text{PT}}$ }as an under-specified POMDP $\mathcal{M}_{\co}^{\text{PT}}(\mathcal{M}_{\co},\cdot)=\left\langle \PT{\sS},\PT{\aAS\agNONE},\PT T_{(\cdot)},\PT R,\PT{\oAS\agNONE},\PT O,\PT{\hor},\PT{\bO}\right\rangle $
with 
\begin{itemize}
\item states of the form $\PTIASPstateT{\ts}=\left\langle \mfCT{\co}{\ts},\oHistGT{\co}\ts\right\rangle $,
\item an underspecified transition model 
\begin{multline*}
\PT{\Tfunc_{(\cdot)}}(\PTIASPstateT{\ts+1}|\PTIASPstateT{\ts},\PTaT{\ts})\defas\\
\Tfunc_{\co}(\mfCT{\co}{\ts+1}|\mfCT{\co}{\ts},\jdrGT{\co}{\ts}(\oHistGT\co\ts),\ifsCT{\co}{\ts+1})\Ofunc_{\co}(\joGT{\co}{\ts+1}|\jdrGT{\co}{\ts}(\oHistGT\co\ts),\mfCT{\co}{\ts+1}),
\end{multline*}

\item and $\PT{\aAS\agNONE},\PT R,\PT{\oAS\agNONE},\PT O,\PT{\hor},\PT{\bO}$
as above.
\end{itemize}
\end{definition}

Since this model is a special case of a POMDP, the theory developed
in \sect\ref{sec:QMPOMDP-based} applies: we can maintain a plan-time
sufficient statistic $\ptssIT\co\ts$ (essentially the `belief'
$\PT b$ over augmented states $\PTIASPstateT{\ts}=\left\langle \mfCT{\co}{\ts},\oHistGT{\co}\ts\right\rangle $)
and we can write down the value function using \eqref{eq:Q-POMDP--back-proj-based}.
Most importantly, the IO back-projection~\eqref{eq:IO-BackProj_POMDP}
also applies, which means that (similar to the MPOMDP case) we can
avoid ever constructing the full PT-IASP. The IO back-projec\-tion
in this case translates to:
\begin{multline}
\nu_{IO}^{\jdrGT{\co}{\ts}}(\mfCT{\co}{\ts},\oHistGT{\co}\ts)\defas\max_{\ifsCT{\co}{\ts+1}}\sum_{\mfCT{\co}{\ts+1}}\Pr(\joGT{\co}{\ts+1}|\jdrGT{\co}{\ts}(\oHistGT\co\ts),\mfCT{\co}{\ts+1})\\
\Pr(\mfnCT{\co}{\ts+1}|\mfCT{\co}{\ts},\jdrGT{\co}{\ts}(\oHistGT\co\ts),\ifsCT{\co}{\ts+1})\Pr(\mflCT{\co}{\ts+1}|\mfCT{\co}{\ts},\jdrGT{\co}{\ts}(\oHistGT\co\ts))\\
\nu_{IO}(\mfCT{\co}{\ts+1},\oHistGT{\co}{\ts+1}).\label{eq:optimBackProj-DecPOMDP}
\end{multline}
Here, we omitted the superscript for the $\mathit{NULL}$ observation.
Also, note that $O(\oAT i{\ts+1}|\aAT i{\ts},\mfAT i{\ts+1})$ in
\eqref{eq:IO-BackProj_POMDP} corresponds to the \emph{$\mathit{NULL}$
}observation in the PT model, but since the observation histories
are in the states, $\Pr(\joGT{\co}{\ts+1}|\jdrGT{\co}{\ts}(\oHistGT\co\ts),\mfCT{\co}{\ts+1})$
comes out of the \emph{transition}~model.

Again, given this modified back-projection, the IO-Q-Dec-POMDP value
$\argsC{\hat{\V}}{\co}^{D}$ can be computed using any \emph{exact}
POMDP solution method that makes use of vector back-projections; all
that is required is to substitute these the back projections by their
modified form \eqref{eq:optimBackProj-DecPOMDP}.

\begin{corollary}\label{cor:IO-Q-Dec-POMDP-is-UB}IO-Q-Dec-POMDP
yields an upper bound to the locally-optimal value: $V_{\co}^{LO}\leq\argsC{\hat{\V}}{\co}^{D}$.\end{corollary}\begin{proof}
Directly by applying \cor\ref{cor:IO-Q-MPOMDP-is-UB} to\emph{ }$\mathcal{M}_{\co}^{\text{PT}}.$\end{proof}

\subsection{Computational Complexity}

\label{sec:complexity-analysis}

Due to the maximization in \eqref{eq:IO-Q-QMMDP}, \eqref{eq:IO-BackProj_POMDP}
and \eqref{eq:optimBackProj-DecPOMDP}, IO back-projections are more
costly than regular (non-IO) back-projections. Here we analyze the
computational complexity of the proposed algorithms relative to the
regular, non-IO, backups.

We start by comparing IQ-Q-MMDP backup operation~\eqref{eq:IO-Q-QMMDP}
to the regular MMDP backup for an SP that does not have incoming influences.
For such an SP, the MMDP backup is given by
\begin{multline}
\negmedspace\negmedspace\negmedspace\negmedspace\Q(\mfCT{\co}{\ts},\jaGT{\co}{\ts})=\sum_{\mfCT{\co}{\ts+1}}\Pr(\mfCT{\co}{\ts+1}|\mfCT{\co}{\ts},\jaGT{\co}{\ts})\\
\Big[\RC{\co}(\mfCT{\co}{\ts},\jaGT{\co}{\ts},\mfCT{\co}{\ts+1})+\max_{\jaGT{\co}{\ts+1}}\Q(\mfCT{\co}{\ts+1},\jaGT{\co}{\ts+1})\Big].\label{eq:QMMDP}
\end{multline}
Comparing \eqref{eq:IO-Q-QMMDP} and \eqref{eq:QMMDP}, we see two
differences: 1) the transition probabilities can be written in one
term since we do not need to discriminate NLAFs from OLAFs, and 2)
there is no maximization over influence sources~$\ifsCT{\co}{\ts+1}$.
The first difference does not induce a change in computational cost,
only in notation: in both cases the entire transition probability
is given as the product of next-stage CPTs. The second difference
does induce a change in computational cost: in \eqref{eq:IO-Q-QMMDP},
in order to select the maximum, the inner part of the right-hand side
needs to be evaluated for each instantiation of the influence sources
$\left|\ifsCT{\co}{\ts+1}\right|$. That is, the computational complexity
of a Q-MMDP backup (for a particular $(\mfCT{\co}{\ts},\jaGT{\co}{\ts})$-pair)
is 
\[
O(\left|\sfacS_{\co}\right|\times\left|\jaG{\co}\right|),
\]
whereas the total computational cost of the IO-Q-MMDP backup is:
\[
O(\left|\ifsCT{\co}{\ts+1}\right|\times\left|\sfacS_{\co}\right|\times\left|\jaG{\co}\right|).
\]
As such, the complexity of each backup is multiplied by the number
of influence source instantiations $\left|\ifsCT{\co}{\ts+1}\right|$.
This means that the overhead of IO-Q-MMDP relative to Q-MMDP is $O(\left|\ifsCT{\co}{\ts+1}\right|)$. 

A similar comparison of the overhead of the IO-Q-MPOMDP back projection
\eqref{eq:IO-BackProj_POMDP} with respect to the regular Q-MPOMDP
back projection \eqref{eq:BackProj_POMDP} leads to exactly the same
relative overhead of $O(\left|\ifsCT{\co}{\ts+1}\right|)$. Since
IO-Q-Dec-POMDP makes use of this last result, also in this case the
relative overhead is $O(\left|\ifsCT{\co}{\ts+1}\right|)$. 

Concluding, the computational complexity of the methods we proposed
to compute IO-UBs is given by multiplying the computational complexity
of the `underlying' MMDP, MPOMDP or Dec-POMDP method with $\left|\ifsCT{\co}{\ts+1}\right|$.
This means that the relative overhead, is equal for all methods. Since
the amount of overhead is a linear function of $\left|\ifsCT{\co}{\ts+1}\right|$,
it is problem dependent: more densely coupled problems will lead to
a higher overhead, while very loosely coupled problems have a low
overhead.

\section{Global Upper Bounds}

\label{sec:global-upper-bounds}

We next discuss how the methods to compute local upper bounds can
be employed in order to compute a global upper bound for factored
Dec-POMDPs.

The basic idea is to apply a non-overlapping decomposition $\mathcal{C}$
(i.e., a partitioning) of the reward functions $\left\{ \RI l\right\} $
of the original factored Dec-POMDP into SPs $\co\in\mathcal{C}$,
and to compute an IO upper bound $\hat{\V}_{\co}^{IO}$ for each (which
can be any of the three IO-UBs proposed in Section \ref{sec:local-upper-bounds}).
Our \emph{global influence-optimistic upper bound }is then given by:
\begin{equation}
\hat{\V}^{IO}\defas\sum_{\co\in\mathcal{C}}\hat{\V}_{\co}^{IO}.\label{eq:global_V^IO}
\end{equation}

\fig\ref{fig:global-upper-bound-construction} illustrates the construction
of a global upper bound~$\hat{\V}$ for the 6-agent \<FFG>, showing
the original problem (top row) and two possible decompositions in
SPs. The second row specifies a decomposition into two SPs, while
the third row uses three SPs. The illustration clearly shows how (in
this problem) a decomposition eliminates certain agents completely
and replaces them with optimistic assumptions: E.g., in the second
row, during the computation of $\hat{\V}_{\co}^{IO}$ for both SPs
($\co=1,2$) the assumption is made that agent 3 will always fight
fire in the SP under concern. Effectively we assume that agent~3
fights fire at both house 3 and house 4 simultaneously (and hence
is represented by a superhero figure). \fig\ref{fig:global-upper-bound-construction}
also illustrates that, due to the line structure of \<FFG>, there
are two \emph{types }of SPs: `internal' SPs which make optimistic
assumptions on two sides, and `edge' SPs that are optimistic at
just one side.

\begin{figure}
\begin{centering}
\textbf{\includegraphics[width=1\columnwidth]{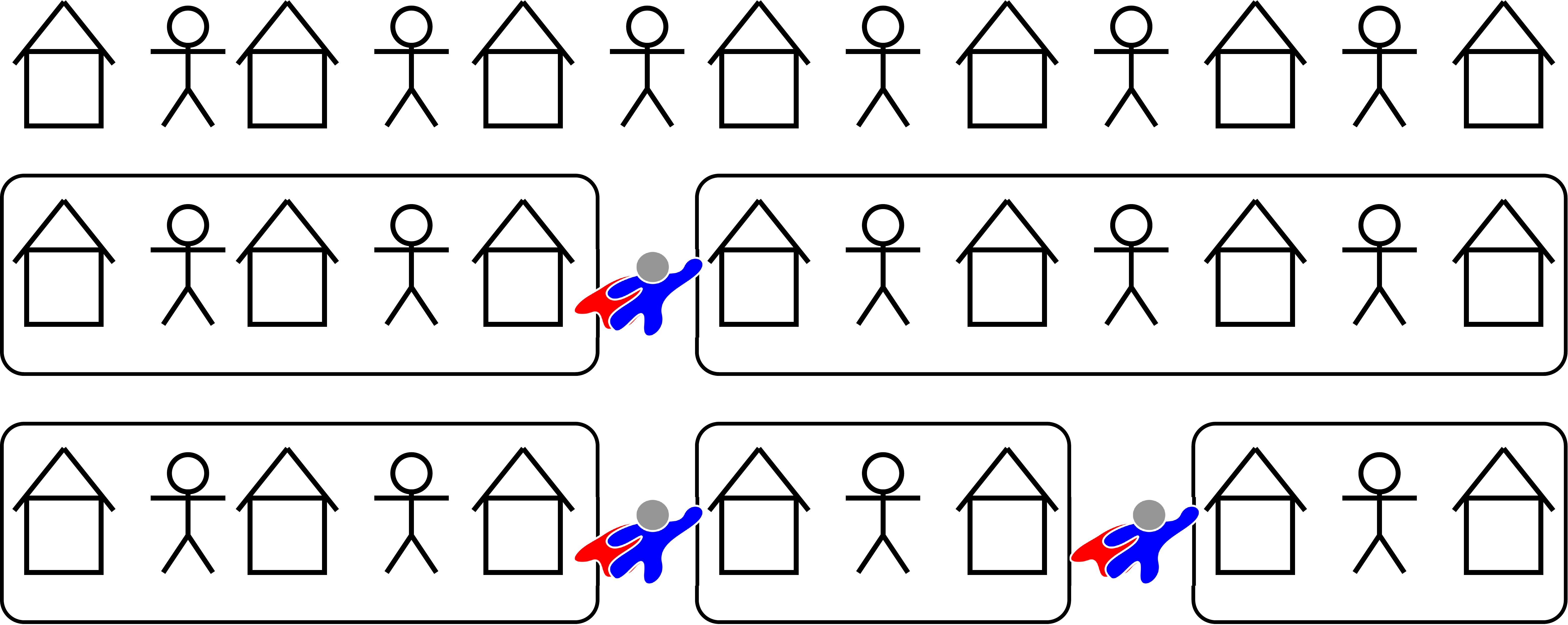}}
\par\end{centering}

\protect\caption{Illustration of construction of a global upper bound for $6$-agent
\<FFG> using influence-optimistic upper bounds for sub-problems.}

\label{fig:global-upper-bound-construction}
\end{figure}

Finally, we formally prove the correctness of our proposed upper bounding
scheme.

\begin{theorem}\label{thm:global-UB}Let $\mathcal{C}$ be a partitioning
of the reward function set $\RS$ into sub-problems such that every
$\RI l$ is represented in one SP $\co$, then the global IO-UB is
in fact an upper bound to the optimal value $\hat{\V}^{IO}\geq V(\jpol^{*})$.\end{theorem}

\begin{proof}Starting from the definition \eqref{eq:global_V^IO},
we have
\begin{eqnarray*}
\hat{\V}^{IO} & \defas & \sum_{\co\in\mathcal{C}}\hat{\V}_{\co}^{IO}\overset{\mathrm{\{\sect\ref{sec:local-upper-bounds}\}}}{\geq}\sum_{\co\in\mathcal{C}}V_{\co}^{LO}\\
 & \defas & \sum_{\co\in\mathcal{C}}\max_{\jpolG{\excl{\co}}}V_{\co}^{BR}(\jpolG{\excl{\co}})\\
 & \defas & \sum_{\co\in\mathcal{C}}\max_{\jpolG{\excl{\co}}}\max_{\jpolG{\co}}\argsC{\V}{\co}(\jpolG{\co},\jpolG{\excl{\co}})\\
 & \geq & \max_{\jpol}\sum_{\co\in\mathcal{C}}\argsC{\V}{\co}(\jpolG{\co},\jpolG{\excl{\co}})\overset{\text{\{\ensuremath{\mathcal{C}}is a partition\}}}{=}V(\jpol^{*})
\end{eqnarray*}
thus proving the result. \end{proof}

\section{Empirical Evaluation }

\label{sec:evaluation}

In order to test the potential impact of the proposed influence-optimistic
upper bounds, we present numerical results in the context of a number
of benchmark problems. In this evaluation, we focus on the (relative)
values found by these heuristics, as we hope that these will spark
a number of interesting ideas for further research (such as the notion
of `influence strength', its relation to approximability of factored
Dec-POMDPs, and the key idea that reasonable bounds for very large
problems may be possible). \textbf{}We do not thoroughly investigate
timing results as the analysis of \sect\ref{sec:complexity-analysis}
indicates that relative timing results follow those of regular (non-IO)
MMDP, MPOMDP and Dec-POMDP methods; see, e.g.,~\cite{Oliehoek08JAIR}
for a comparison of such timing results.\textbf{ }However, in order
to provide a overall idea of the run times, we do provide some indicative
running times.%
\footnote{All experiments are run on a Intel Xeon E5-2650L, 32GB system making
use of one core only.%
}

\begin{figure*}
\begin{centering}
\providecommand{\trh}{2.6cm}
\includegraphics[height=\trh]{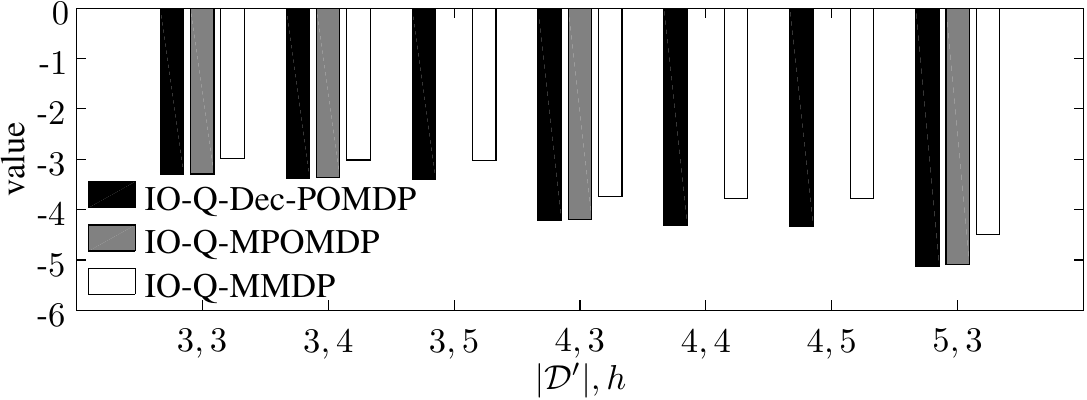}
\includegraphics[height=\trh]{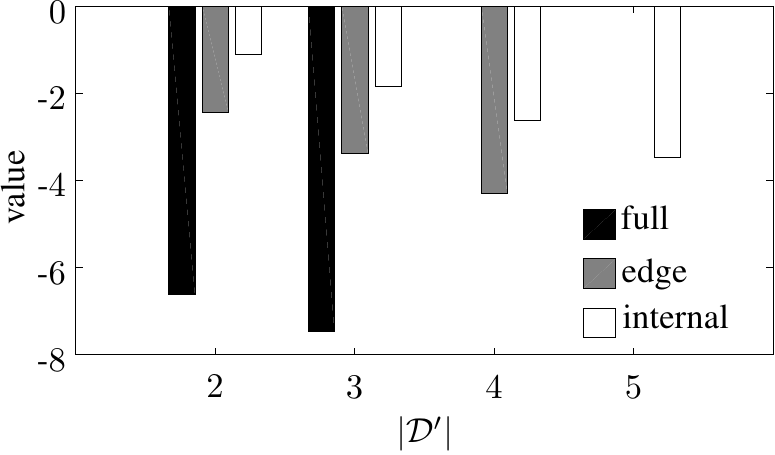}
\includegraphics[height=\trh]{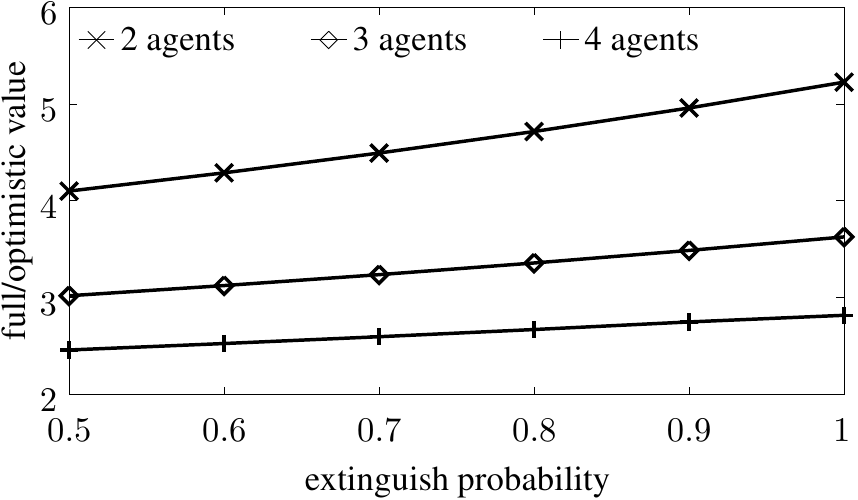}\\[-.5em]
\par\end{centering}

\protect\caption{Left: Different IO-UBs on `edge' \<FFG> problems. Middle: IO-Q-Dec-POMDP
upper bound on different types (internal, edge, full) of SPs. Right:
The impact the `influence strength'.}

\label{fig:result-row}

\label{fig:IO-heuristics-compared}\label{fig:SP-types}\label{fig:influence-strength}
\end{figure*}

\subsection{Comparison of Different Bounds}

The bounds that we propose are ordered in tightness, $\argsC{\hat{\V}}{\co}^{D}\leq\argsC{\hat{\V}}{\co}^{P}\leq\argsC{\hat{\V}}{\co}^{M}$,
similar to how regular (non-IO) Dec-POMDP, Q-MPOMDP, and Q-MMDP values
relate~\cite{Oliehoek08JAIR}. To get an understanding of how these
differences turn out in practice, \fig\ref{fig:IO-heuristics-compared}(left)
compares the different upper bounds introduced. 

Although the approach described in the paper is general, in the numerical
evaluation here we exploit the property that the optimistic influences
are easily identified off-line, which allows for the construction
of small `optimistic Dec-POMDPs' (respectively MPOMDPs or MMDPs)
without sacrificing in bound quality. E.g., in order to compute the
local IQ-Q-Dec-POMDP upper bound for a 3-house \<FFG> `edge' SP,
we define a regular 3-house Dec-POMDP where the transitions probabilities
for the first house (say $\sfacI i$ in \fig\ref{fig:sub-problem})
are modified to account for the optimistic assumption that another
(superhero) agent fights fire there and that its neighbor is not burning
(i.e., $\aA{i-1}=right$ and $\sfacI{i-1}=not\_burning$ in \fig\ref{fig:sub-problem}).
The values $\argsC{\hat{\V}}{\co}^{D}$ (resp. $\argsC{\hat{\V}}{\co}^{P},\argsC{\hat{\V}}{\co}^{M}$)
are computed by running a state-of-the-art Dec-POMDP solver~\cite{Oliehoek13JAIR}
(resp.\ incremental pruning~\cite{Cassandra97UAI}, plain dynamic
programming) on such optimistically defined problems. 

\fig\ref{fig:IO-heuristics-compared}(left) shows the values for
such `edge' problems. Missing bars indicate time-outs~(>4h). As
an indication of run time, the $\nrASP=3,\hor=5$ problem took $2.21s$
for IO-Q-MMDP, and $995.28s$ for IO-Q-Dec-POMDP. The shown values
indicate that $\argsC{\hat{\V}}{\co}^{D},\argsC{\hat{\V}}{\co}^{P}$
can be tighter than $\argsC{\hat{\V}}{\co}^{M}$ in practice. In most
cases, the difference between IO-Q-MPOMDP and IO-Q-Dec-POMDP is small,
but these could become larger for longer horizons~\cite{Oliehoek08JAIR}.
We performed the same analysis for the \<Aloha> benchmark~\cite{Oliehoek13AAMAS},
and found very similar results.

We also compare the bounds found on the different types of SPs (internal
and edge-cases, see \fig\ref{fig:global-upper-bound-construction})
encountered in \<FFG> ($\hor=4$). In addition, \fig\ref{fig:SP-types}(middle)
also includes---if computable within the allowed time---values of
SPs that are `full' problems (i.e., the regular optimal Dec-POMDP
value for the full \<FFG> instance with the indicated number of agents.)
This makes clear that the optimistic assumption has quite some effect:
being optimistic at one edge more than halves the optimal cost, and
the IO assumption at both edges of the SP leads to another significant
reduction of that cost. This is to be expected: the optimistic problems
assume that there \emph{always }will be another agent fighting fire
at the house at an optimistic edge, while the full problem \emph{never}
has another agent at that same house. When also taking into account
the transition probabilities---two agents at a house will completely
extinguish a fire---it is clear that the IO assumption should have
a high impact on the local value.

\subsection{The Effect of Influence Strength}

\fig\ref{fig:SP-types}(middle) makes clear that the IO assumption
in \<FFG> is quite strong and leads to a significant over-estimation
of the local value compared to the `full' problem. We say that \<FFG> 
has a high \emph{influence strength. }In fact, this hints at a new
dimension of the qualification of weak coupling~\cite{Witwicki11AAMAS}
that takes into account the variance in NLAF probabilities (as a function of the change of value of the influence source) 
and their impact on the local value.

As a preliminary investigation of this concept, we devise a modification
of \<FFG> where the influence strength can be controlled. In particular,
we parameterize the probability that a fire is extinguished completely
when 2 agents visit the same house, which is set to 1 in the original
problem definition. Lower values of this probability mean that optimistically
assuming there is another agent at a house will lead to less advantage,
and thus lower influence strength.

\fig\ref{fig:influence-strength}(right) shows the results of this
experiment. It shows that there is a clear relation between the fire-extinguish
probability when two agents fight fire at a house, and the ratio between
the `regular' (Dec-POMDP) value and optimistic value. It also shows
that SPs with more agents are less affected: this makes sense since
optimistic assumptions account for a smaller fraction of the achievable
value. In other words, larger sub-problems give a tighter approximation.

\subsection{Bounding Heuristic Methods}

Here we investigate the ability to provide informative \emph{global
}upper bounds. While the previous analysis shows that the overestimation
is quite significant at the \emph{true }edges of the problem (where
no agents exist), this is not necessarily informative of the overestimation
at internal edges in decompositions of larger problems (where other
agents do exist, even if not superheros). As such, besides investigating
the upper bounding capability, the analysis here also provides a better
understanding of such internal overestimations.

We use the tightest upper bound we could find by considering different
SP partitions, with sizes ranging from $\nrA=2$--$5$, and investigate
the guarantees that it can provide for transfer planning (TP)~\cite{Oliehoek13AAMAS},
which is one of the methods capable of providing solutions for large
factored Dec-POMDPs. Since the method is a heuristic method that does
not provide the exact value of the reported joint policy, the value
of TP, $V^{\mathit{TP}}$, is determined using $10.000$ simulations
of the found joint policy leading to accurate estimates.%
\footnote{Note that there is no method for Dec-POMDP policy evaluation that
runs in polynomial time. In fact, existence of such a method would
reduce the complexity of solving a Dec-POMDP to NP, an impossibility
since the time hierarchy theorem implies that NP$\neq$ NEXP.%
} To put the results into context, we also show the value of a random
policy. Finally, we show (second y-axis in Fig. \ref{fig:TP-UB-many-agents})
what we call the \emph{empirical approximation factor (EAF)}:
\[
EAF=\max\{\frac{\hat{V}^{\mathit{IO}}}{V^{\mathit{TP}}},\frac{V^{\mathit{TP}}}{\hat{V}^{\mathit{IO}}}\}.
\]
This is a number comparable to the approximation factors of \emph{approximation
algorithms~}\cite{Vazirani01book}.%
\footnote{`Empirical' emphasizes the lack of \emph{a priori }guarantees.%
}

\begin{figure}
\begin{centering}
\includegraphics[width=1\columnwidth]{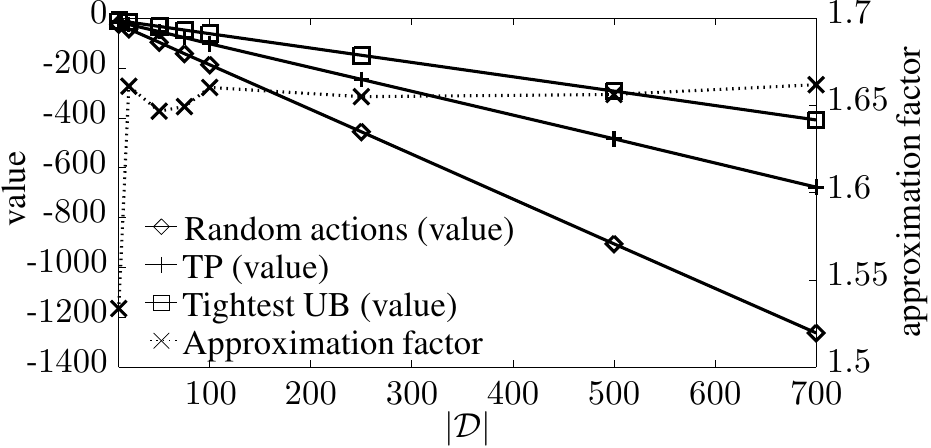}\\[-0.5em]
\par\end{centering}

\protect\caption{The global IO-Q-Dec-POMDP upper bound on large \<FFG> instances.}

\label{fig:TP-UB-many-agents}
\end{figure}

We computed upper bounds for large, horizon $\hor=4$, \<FFG> instances.
The computation of the local upper bounds for the largest SPs used
(i.e., $\nrASP=5$) took 3.31 secs for IO-Q-MMDP and 2696.23s for
IO-Q-Dec-POMDP. \fig\ref{fig:TP-UB-many-agents} shows the results
that indicate that the upper bound is relatively tight: the solutions
found by TP are not too far from the upper bound. In particular, the
EAF lies typically between 1.4 and 1.7, thus providing firm guarantees
for solutions of factored Dec-POMDPs with up to 700 agents. Moreover,
we see that the EAF stays roughly constant for the larger problem
instances indicating that \emph{relative }guarantees do not degrade
as the number of agents increase. Of course, the question of whether
the optimal value lies closer to the blue (UB) or orange (TP) line
remains open; only further research on improved (heuristic) solution
methods and tighter upper bounds can answer that question. However,
we have gone from a situation where the only upper bound we had was
`predict $R_{max}$ for every stage' (which corresponds to the value
0 and EAF=$\infty$) to a situation where we have a much more informative
bound. 

Results obtained for a similar approach for \<Aloha> using SPs containing
up to $\nrASP=6$ agents are shown in \tab\ref{tab:Aloha-TP}. The
numbers clearly illustrate that it is possible to provide very strong
guarantees for problems up to $\nrA=250$ agents (beyond which memory
forms the bottleneck for TP); the solution for the $\nrA=50$ instance
is essentially optimal, indicating also a very tight bound for this
problem.

\begin{table}
\begin{centering}
{\small%
\begin{tabular}{ccccc}
\toprule 
$\nrA$ & 50 & 75 & 100 & 250\tabularnewline
\midrule
$V^{TP}$ & $-71.99$ & $-111.07$ & $-148.70$ & $-382.47$\tabularnewline
$\hat{V}^{IO}$ & $-72.00$ & $-107.06$ & $-144.00$ & $-360.00$\tabularnewline
EAF & $1.00$ & $1.04$ & $1.03$ & $1.06$\tabularnewline
\bottomrule
\end{tabular}}
\par\end{centering}

\protect\caption{Empirical approximation factors for \<Aloha> ($\hor=3$) with varying
number of agents.}

\label{tab:Aloha-TP}

\end{table}

\subsection{Improved Heuristic Influence Search}

Aside from analyzing the solution quality of approximate methods,
our bounds can also be used in optimal methods. In particular, A{*}-OIS~\cite{Witwicki12AAMAS}
solves TD-POMDPs (a sub-class of factored Dec-POMDPs) by decomposing
them into 1-agent SPs, searching through the space of influences,
and pruning using optimistic heuristics. However, existing A{*}-OIS
heuristics treat the unspecified-influence stages of the  SPs as
fully-observable. In contrast, IO-Q-MPOMDP models the \emph{partial
observability} of the SPs.

We now present results that suggest an added computational benefit
to treating partial observability in the A{*}-OIS heuristics. Table
\ref{tab:astar-heuristics-comparison} illustrates the differences
in pruning afforded by four different A{*}-OIS heuristics: \emph{M}
is the baseline MDP-based heuristic from \cite{Witwicki12AAMAS},
\emph{P} is shorthand for IO-Q-MPOMDP, and \emph{M'} and \emph{P'}
are variations that improve tightness with locally-derived probability
bounds on the optimistic influences. We report node counts and runtime
across several problem instances from the HouseSearch domain~\cite{Witwicki12AAMAS}.

As shown, the POMDP-based heuristics tends to allow for more pruning
(fewer expanded nodes) and thereby reduced computation in comparison
to their MDP-based counterparts. Contrasting this reduction across
two variants of \emph{Diamond} HouseSearch suggests that the IO-Q-MPOMDP
heuristics gain more advantage when observability is more restricted.
However, this advantage is sometimes outweighed by the increased computational
overhead of a more complex heuristic calculation (such as in \emph{Squares}
HouseSearch $\hor=4$). The question of predicting in advance when
this will be the case is interesting, but difficult (and orthogonal
to our contribution).

\begin{table}[tb]
\begin{centering}
{\scriptsize
\renewcommand{\tabcolsep}{0mm} 
\renewcommand{\arraystretch}{1}\global\long\def\hs{\hspace{4mm}}
 \global\long\def\hss{\hspace{2mm}}
\begin{tabular}{@{}r@{\hs}l@{\hss}l@{}c@{\hs}l@{\hss}l@{}c@{\hs}l@{\hss}l@{}c@{\hs}l@{\hss}l@{}}
\toprule 
 & \multicolumn{2}{c}{$\hor$=2} &  & \multicolumn{2}{c}{$\hor$=3} &  & \multicolumn{2}{c}{$\hor$=4} &  & \multicolumn{2}{c}{$\hor$=5}\tabularnewline
\midrule 
\multicolumn{12}{c}{\textsc{HouseSearch} : \textit{Diamond} : $Pr(\text{correct obs.})=0.75$ }\tabularnewline
\midrule 
 & \#  & $s$  &  & \#  & $s$  &  & \#  & $s$  &  & \#  & $s$ \tabularnewline
M & \emph{55 } & 2.36  &  & \emph{245 } & 33.86  &  & \emph{2072 } & 1179  &  &  & \tabularnewline
P & \emph{55 } & 2.42  &  & \emph{241 } & 34.21  &  & \emph{1744 } & 1010  &  &  & \tabularnewline
M' & \emph{45 } & 2.00  &  & \emph{91 } & 9.00  &  & \emph{183 } & 64.32  &  & \emph{97 } & 991.2 \tabularnewline
P'  & \emph{45 } & 2.02  &  & \emph{91 } & 10.05  &  & \emph{119 } & 39.84  &  & \emph{66 } & 607.9 \tabularnewline
\midrule 
\multicolumn{12}{c}{\textsc{HouseSearch} : \textit{Diamond} : $Pr(\text{correct obs.})=0.5$ }\tabularnewline
\midrule 
 & \#  & $s$  &  & \#  & $s$  &  & \#  & $s$  &  & \#  & $s$ \tabularnewline
M & \emph{47 } & 3.35  &  & \emph{225 } & 51.80  &  & \emph{1434 } & 1935  &  &  & \tabularnewline
P & \emph{47 } & 3.32  &  & \emph{223 } & 52.75  &  & \emph{1416 } & 1905  &  &  & \tabularnewline
M' & \emph{23 } & 1.96  &  & \emph{24 } & 4.69  &  & \emph{41 } & 38.42  &  & \emph{148 } & 1882 \tabularnewline
P'  & \emph{11 } & 0.89  &  & \emph{15 } & 3.90  &  & \emph{19 } & 25.32  &  & \emph{23 } & 256.8 \tabularnewline
\midrule 
\multicolumn{12}{c}{\textsc{HouseSearch} : \textit{Squares} : $Pr(\text{correct obs.})=0.75$ }\tabularnewline
\midrule 
 & \#  & $s$  &  & \#  & $s$  &  & \#  & $s$  &  &  & \tabularnewline
M & \emph{25 } & 2.41  &  & \emph{311 } & 120.5  &  & \emph{1323 } & 5846  &  &  & \tabularnewline
P & \emph{25 } & 2.56  &  & \emph{311 } & 121.3  &  & \emph{1032 } & 7184  &  &  & \tabularnewline
M' & \emph{25 } & 2.61  &  & \emph{277 } & 109.6  &  & \emph{816 } & 3924  &  &  & \tabularnewline
P'  & \emph{25 } & 2.53  &  & \emph{240 } & 95.28  &  & \emph{752 } & 6760  &  &  & \tabularnewline
\bottomrule
\end{tabular}}
\par\end{centering}

\noindent \raggedright{}\protect\caption{Node count and run time for A{*}-OIS.}
\label{tab:astar-heuristics-comparison} 
\end{table}

\section{Related Work}

\label{sec:related-work}

Here we provide an overview of related approaches.

\paragraph{Scalable Heuristic Methods}

In recent years, many scalable approaches without guarantees have
been developed for Dec-POMDPs and related models \cite{Kumar09AAMAS,Wu10AAAI,Yin11IJCAI,Kumar11IJCAI,Velagapudi11AAMAS,Varakantham12AAAI,Oliehoek13AAMAS,Wu13IJCAI,Varakantham14AAAI}.
The upper bounding mechanism that we propose could be useful for benchmarking
many of these methods.

\paragraph{Sub-Problems vs. Source Problems}

The use of sub-problems (SPs) is conceptually similar to the use of
\emph{source problems }in transfer planning (TP) \cite{Oliehoek13AAMAS}.
Differences are that, in our partitioning-based upper bounding scheme,
SPs are selected such that they do not contain overlapping rewards,
while TP allows for overlapping source problems. Moreover TP does
not consider optimistic influences but implicitly assumes arbitrary
influences. Finally, TP is used as a way to compute a heuristic for
the original problem; our approach here simply returns a scalar value
(although extensions to use IO heuristics for heuristic search for
factored Dec-POMDPs are an interesting direction for future research).

\paragraph{Optimism with Respect to Influences}

Being optimistic with respect to external influences has been considered
before. For instance, Kumar \& Zilberstein \cite{Kumar09AAMAS} make
optimistic assumptions on transitions in an ND-POMDP to derive an
MMDP-based policy which is used to sample belief points for memory-bounded
dynamic programming. The approach does not use these assumptions to
upper bound the global value and the formulation is specific to ND-POMDPs.
As described in the experiments, Witwicki, Oliehoek \& Kaelbling \cite{Witwicki12AAMAS}
use a local upper bound in order to perform heuristic influence search
for TD-POMDPs. IO-Q-MMDP can be seen as a generalization of that heuristic
to both factored Dec-POMDPs and multiagent sub-problems, while our
other heuristics additionally deal with partial observability.

\paragraph{Quality Guarantees for Large Dec-POMDPs}

As mentioned in the introduction, a few approaches to computing upper
bound for large-scale MASs and employing them in heuristic search
methods have been proposed. In particular, there are some scalable
approaches with guarantees for the case of transition and observation
independence as encountered in TOI-MDPs and ND-POMDPs \cite{Becker03AAMAS,Becker04AAMAS,Varakantham07AAMAS,Marecki08AAMAS,Dibangoye13AAMAS,Dibangoye14AAMAS}.
The approach by Witwicki for TD-POMDPs \cite[Sect. 6.6]{Witwicki11PhD}
is a bit more general in that it allows some forms of transition dependence,
as long as the interactions are \emph{directed} (one agent can affect
another) and no two agents can affect the same factor in the same
stage. In addition, scalability relies on each agent having only a
handful of `interaction ancestors'. 

However, these previous approaches rely on the true value function
being factored as the sum of a set \emph{$\set{E}$} of \emph{local
}value\emph{ }components: 
\[
\V(\jpol)=\sum_{e\in\set{E}}\argsC{\V}e(\jpolG e),
\]
where $\jpolG e$ is the \emph{local} joint policy of the agents that
participate in component $e$. (This is also referred to as the `value
factorization' framework~\cite{Kumar11IJCAI}). For this setting,
an upper bound is easily constructed as the sum of local upper bounds:
\[
\hat{\V}(\jpol)=\sum_{e\in\set{E}}\hat{\argsC{\V}e}(\jpolG e).
\]
While this resembles our IO upper bound~\eqref{eq:global_V^IO},
the crucial distinction is that for value factorized settings computing
$\hat{\argsC{\V}e}$ does not require any influence-optimism: the
reason that value-factorization holds is precisely \emph{because there
are no influence} \emph{sources }for the components. As such, our
influence-optimistic upper bounds can be seen as a strict generalization
of the upper bounds that have been employed for settings with factored
value functions. Investigating if such methods, such as the method
by Dibangoye et al\@.~\cite{Dibangoye14AAMAS} can be modified to
use our IO-UBs is an interesting direction of future work.\textbf{}

Finally, the event-detecting multiagent MDP \cite{Kumar09IJCAI} provides
quality guarantees for a specific class of sensor network problems
by using the theory of submodular function maximization. It is the
only previous method with quality guarantees that delivers scalability
with respect to the number of agents without assuming that the value
function of the problem is additively factored into small components.

\section{Conclusions}

\label{sec:Conclusions} We presented a family of \emph{influence-optimistic}
upper bounds for the value of sub-problems of factored Dec-POMDPs,
together with a partition-based decomposition approach that enables
the computation of global upper bounds for very large problems. The
approach builds upon the framework of influence-based abstraction~\cite{Oliehoek12AAAI_IBA},
but---in contrast to that work---makes optimistic assumptions on the
incoming `influences', which makes the sub-problems easier to solve.
An empirical evaluation compares the proposed upper bounds and demonstrates
that it is possible to achieve guarantees for problems with hundreds
of agents, showing that found heuristic solution are in fact close
to optimal (empirical approximation factors of $<1.7$ in all cases
and sometimes substantially better). This is a significant contribution,
given the complexity of computing $\epsilon$-approximate solutions
and the fact that tight global upper bounds are of crucial importance
to interpret the quality of heuristic solutions.

Intuitively, the proposed approach is expected to work well in settings
where the Dec-POMDP is `weakly' coupled. The work by Witwicki \&
Durfee~\cite{Witwicki11AAMAS} identifies three dimensions that can
be used to quantify the notion of weak coupling. Our experiments suggest
the existence of a new dimension that can be thought of as \emph{influence
strength}. This dimension captures the impact of non-local behavior
on local \emph{values }and thus directly relates to how well a problem
can be approximated using localized components.

In this paper we focused on the finite-horizon case, but the principle
of influence optimism underlying the upper-bounding approach can also
be applied in infinite-horizon settings. Furthermore, the approach
can be modified to compute `pessimistic' influence (i.e., lower)
bounds, which could be useful in competitive settings, or for risk-sen\-si\-tive
planning \cite{marecki10AAMAS}. It is also immediately applicable
to problems involving `unpredictable' dynamics \cite{Witwicki13ICAPS,delgado11AIJ}.
Finally, our upper-bounding method contributes a useful precursor
for techniques that automatically search the space of possible upper
bounds decompositions, efficient optimal influence-space heuristic
search methods (for which we provided preliminary evidence in this
paper), and A{*} methods for a large class of factored Dec-POMDPs.
In particular, a promising idea is to employ our factored upper bounds
in combination with the heuristic search methods by Dibangoye et al\@.~\cite{Dibangoye14AAMAS}.
While it is not possible to directly use that method since it additionally
requires a factored lower bound function, pessimistic-influence bounds
could provide those.

A limitation of the current approach is that sub-prob\-lems still
need to be relatively small, since we rely on \emph{optimal }optimistic
solution of the sub-problems. Developing more scalable `optimistic
solution methods' thus is an important direction of future work.
Experiments with influence search indicate that using probabilistic
bounds on the positive influences has a major impact~\cite{Witwicki12AAMAS}.
As such, another important direction of future work is investigating
if it is possible to develop tighter upper bounds by making more realistic
optimistic assumptions.
\paragraph{Acknowledgments}

\noindent F.O. is supported by NWO Innovational Research Incentives
Scheme Veni \#639.021.336.

\bibliographystyle{abbrv}
\bibliography{AAMAS15-UBs}

\onecolumn

\appendix

\section{Proofs}

\label{sec:app:proofs}

\lem\ref{lem:maxD_v__leq__vIO}. Let $\jpolG{\co}^{t:h-1}$ be a
$(h-\ts)$-steps-to-go policy. Let $\nu\in\mathcal{V}$ and $\nu_{IO}\in\mathcal{V}_{IO}$
be the vectors induced by $\jpolG{\co}^{t:h-1}$ under regular MPOMDP
back-projections (for some $\ifpiC{\co}$), and under IO back-projections.
Then 
\[
\forall_{\mfCT{\co}{\ts}}\qquad\max_{\dsetCT{\co}{\ts+1}}\nu(\langle\mfCT{\co}{\ts},\dsetCT{\co}{\ts+1}\rangle)\leq\nu_{IO}(\mfCT{\co}{\ts}).
\]

\begin{proof}

The proof is via induction. The base case is for the last stage $t=h-1$,
in which case the vectors only consist of the immediate reward:
\[
\max_{\dsetCT{\co}{\ts+1}}\nu(\langle\mfCT{\co}{\ts},\dsetCT{\co}{\ts+1}\rangle)=\max_{\dsetCT{\co}{\ts+1}}R_{\ja^{\ts}}(\mfCT{\co}{\ts})=R_{\ja^{\ts}}(\mfCT{\co}{\ts})=\nu_{IO}(\mfCT{\co}{\ts}),
\]
where $\ja^{\ts}$ is the joint action specified by $\jpolG{\co}^{h-1:h-1}$,
thus proving the base case. The induction step follows.

\textbf{Induction Hypothesis: }Suppose that, given that $\nu\in\mathcal{V}$
and $\nu_{IO}\in\mathcal{V}_{IO}$ are vectors for the same policy
$\jpolG{\co}^{t+1:h-1}$, for all $\mfCT{\co}{\ts+1}$
\begin{equation}
\max_{\dsetCT{\co}{\ts+2}}\nu(\langle\mfCT{\co}{\ts+1},\dsetCT{\co}{\ts+2}\rangle)\leq\nu_{IO}(\mfCT{\co}{\ts+1})\label{eq:IS:IH}
\end{equation}
holds.

\textbf{To prove: }Given that $\nu\in\mathcal{V}$ and $\nu_{IO}\in\mathcal{V}_{IO}$
are vectors for the same policy $\jpolG{\co}^{t:h-1}$, for all $\mfCT{\co}{\ts}$

\begin{equation}
\max_{\dsetCT{\co}{\ts+1}}\nu(\langle\mfCT{\co}{\ts},\dsetCT{\co}{\ts+1}\rangle)\leq\nu_{IO}(\mfCT{\co}{\ts}).\label{eq:IS:ToProof}
\end{equation}

\textbf{Proof:}

We first define the vectors from the l.h.s. Let $\jaT{\ts}$ denote
the first joint action specified by~$\jpolG{\co}^{t:h-1}$. Then
we can write
\begin{equation}
\nu(\langle\mfCT{\co}{\ts},\dsetCT{\co}{\ts+1}\rangle)=r_{a}(\mfCT{\co}{\ts})+\gamma\sum_{\joT{\ts+1}}\nu^{\pi|ao}(\langle\mfCT{\co}{\ts},\dsetCT{\co}{\ts+1}\rangle)\label{eq:IS:vec-if__1}
\end{equation}
where $\nu^{\pi|ao}$ is the back-projection of the vector $\Gamma(\jpolG{\co}^{t:h-1},\jaGT{\co}{\ts},\joGT{\co}{\ts+1})$
that corresponds to $\jpolG{\co}^{t+1:h-1}=\jpolG{\co}^{t:h-1}\Downarrow_{a,o}$
(the sub-tree of $\jpolG{\co}^{t:h-1}$ given $\jaGT{\co}{\ts},\joGT{\co}{\ts+1}$).
That is, by filling out the definition of back-projection, we get
\[
\nu(\langle\mfCT{\co}{\ts},\dsetCT{\co}{\ts+1}\rangle)=r_{a}(\mfCT{\co}{\ts})+\gamma\sum_{\joT{\ts+1}}\left[\sum_{\mfCT{\co}{\ts+1}}\bar{O}(\joT{\ts+1}|\jaT{\ts},\mfCT{\co}{\ts+1})\bar{T}_{\ifpiC{\co}}(\mfCT{\co}{\ts+1}|\langle\mfCT{\co}{\ts},\dsetCT{\co}{\ts+1}\rangle,\jaGT{\co}{\ts})[\Gamma(\jpolG{\co}^{t:h-1},\jaGT{\co}{\ts},\joGT{\co}{\ts+1})](\langle\mfCT{\co}{\ts+1},\dsetCT{\co}{\ts+2}\rangle)\right],
\]
where $\dsetCT{\co}{\ts+2}$ is specified as a function of $\mfCT{\co}{\ts},\dsetCT{\co}{\ts+1},\jaT{\ts},\mfCT{\co}{\ts+1}$.
Clearly, introducing a maximization can not decrease the value, so 

\begin{eqnarray}
\nu(\langle\mfCT{\co}{\ts},\dsetCT{\co}{\ts+1}\rangle) & \leq & r_{a}(\mfCT{\co}{\ts})+\gamma\sum_{\joT{\ts+1}}\sum_{\mfCT{\co}{\ts+1}}\bar{O}(\joT{\ts+1}|\jaT{\ts},\mfCT{\co}{\ts+1})\bar{T}_{\ifpiC{\co}}(\mfCT{\co}{\ts+1}|\langle\mfCT{\co}{\ts},\dsetCT{\co}{\ts+1}\rangle,\jaGT{\co}{\ts})\max_{\dsetCT{\co}{\ts+2}}[\Gamma(\jpolG{\co}^{t:h-1},\jaGT{\co}{\ts},\joGT{\co}{\ts+1})](\langle\mfCT{\co}{\ts+1},\dsetCT{\co}{\ts+2}\rangle)\nonumber \\
\text{\{I.H.\}} & \leq & r_{a}(\mfCT{\co}{\ts})+\gamma\sum_{\joT{\ts+1}}\sum_{\mfCT{\co}{\ts+1}}\bar{O}(\joT{\ts+1}|\jaT{\ts},\mfCT{\co}{\ts+1})\bar{T}_{\ifpiC{\co}}(\mfCT{\co}{\ts+1}|\langle\mfCT{\co}{\ts},\dsetCT{\co}{\ts+1}\rangle,\jaGT{\co}{\ts})[\Gamma_{IO}(\jpolG{\co}^{t:h-1},\jaGT{\co}{\ts},\joGT{\co}{\ts+1})](\mfCT{\co}{\ts+1})\label{eq:IS:vec-if__2}
\end{eqnarray}
where $\Gamma_{IO}(\jpolG{\co}^{t:h-1},\jaGT{\co}{\ts},\joGT{\co}{\ts+1})$
is the IO vector that corresponds to $\jpolG{\co}^{t+1:h-1}=\jpolG{\co}^{t:h-1}\Downarrow_{a,o}$
.

Now we define the r.h.s. vector:
\begin{multline}
\nu_{IO}(\mfCT{\co}{\ts})=r_{a}(\mfCT{\co}{\ts})+\gamma\sum_{\joT{\ts+1}}\nu_{IO}{}^{\pi|ao}(\mfCT{\co}{\ts})\\
=r_{a}(\mfCT{\co}{\ts})+\gamma\sum_{\joT{\ts+1}}\left[\max_{\ifsCT{\co}{\ts+1}}\sum_{\mfCT{\co}{\ts+1}}O(\joGT{\co}{\ts+1}|\jaG{\co},\mfCT{\co}{\ts+1})\Pr(\mfnCT{\co}{\ts+1}|\mfCT{\co}{\ts},\jaG{\co},\ifsCT{\co}{\ts+1})\Pr(\mflCT{\co}{\ts+1}|\mfCT{\co}{\ts},\jaG{\co})[\Gamma_{IO}(\jpolG{\co}^{t:h-1},\jaGT{\co}{\ts},\joGT{\co}{\ts+1})](\mfCT{\co}{\ts+1})\right]\label{eq:IS:vec-IO__1}
\end{multline}
We need to show that 
\begin{multline*}
\max_{\dsetCT{\co}{\ts+1}}\left[r_{a}(\mfCT{\co}{\ts})+\gamma\sum_{\joT{\ts+1}}\sum_{\mfCT{\co}{\ts+1}}\bar{O}(\joT{\ts+1}|\jaT{\ts},\mfCT{\co}{\ts+1})\bar{T}_{\ifpiC{\co}}(\mfCT{\co}{\ts+1}|\langle\mfCT{\co}{\ts},\dsetCT{\co}{\ts+1}\rangle,\jaGT{\co}{\ts})[\Gamma_{IO}(\jpolG{\co}^{t:h-1},\jaGT{\co}{\ts},\joGT{\co}{\ts+1})](\mfCT{\co}{\ts+1})\right]\\
\leq r_{a}(\mfCT{\co}{\ts})+\gamma\sum_{\joT{\ts+1}}\max_{\ifsCT{\co}{\ts+1}}\sum_{\mfCT{\co}{\ts+1}}O(\joGT{\co}{\ts+1}|\jaG{\co},\mfCT{\co}{\ts+1})\Pr(\mfnCT{\co}{\ts+1}|\mfCT{\co}{\ts},\jaG{\co},\ifsCT{\co}{\ts+1})\Pr(\mflCT{\co}{\ts+1}|\mfCT{\co}{\ts},\jaG{\co})[\Gamma_{IO}(\jpolG{\co}^{t:h-1},\jaGT{\co}{\ts},\joGT{\co}{\ts+1})](\mfCT{\co}{\ts+1})
\end{multline*}
which holds if and only if
\begin{multline}
\max_{\dsetCT{\co}{\ts+1}}\sum_{\joT{\ts+1}}\sum_{\mfCT{\co}{\ts+1}}\bar{O}(\joT{\ts+1}|\jaT{\ts},\mfCT{\co}{\ts+1})\bar{T}_{\ifpiC{\co}}(\mfCT{\co}{\ts+1}|\langle\mfCT{\co}{\ts},\dsetCT{\co}{\ts+1}\rangle,\jaGT{\co}{\ts})[\Gamma_{IO}(\jpolG{\co}^{t:h-1},\jaGT{\co}{\ts},\joGT{\co}{\ts+1})](\mfCT{\co}{\ts+1})\\
\leq\sum_{\joT{\ts+1}}\max_{\ifsCT{\co}{\ts+1}}\sum_{\mfCT{\co}{\ts+1}}O(\joGT{\co}{\ts+1}|\jaG{\co},\mfCT{\co}{\ts+1})\Pr(\mfnCT{\co}{\ts+1}|\mfCT{\co}{\ts},\jaG{\co},\ifsCT{\co}{\ts+1})\Pr(\mflCT{\co}{\ts+1}|\mfCT{\co}{\ts},\jaG{\co})[\Gamma_{IO}(\jpolG{\co}^{t:h-1},\jaGT{\co}{\ts},\joGT{\co}{\ts+1})](\mfCT{\co}{\ts+1}).\label{eq:IS:to-proof_2}
\end{multline}
To show this is the case, we start with the l.h.s.:
\begin{align*}
 & \max_{\dsetCT{\co}{\ts+1}}\sum_{\joT{\ts+1}}\sum_{\mfCT{\co}{\ts+1}}\bar{O}(\joT{\ts+1}|\jaT{\ts},\mfCT{\co}{\ts+1})\bar{T}_{\ifpiC{\co}}(\mfCT{\co}{\ts+1}|\langle\mfCT{\co}{\ts},\dsetCT{\co}{\ts+1}\rangle,\jaGT{\co}{\ts})[\Gamma_{IO}(\jpolG{\co}^{t:h-1},\jaGT{\co}{\ts},\joGT{\co}{\ts+1})](\mfCT{\co}{\ts+1})\\
= & \max_{\dsetCT{\co}{\ts+1}}\sum_{\joT{\ts+1}}\sum_{\mfCT{\co}{\ts+1}}\bar{O}(\joT{\ts+1}|\jaT{\ts},\mfCT{\co}{\ts+1})\left[\sum_{\ifsCT{\co}{\ts+1}}\Pr(\mfnAT i{\ts+1}|\mfCT{\co}{\ts},\jaGT{\co}{\ts},\ifsCT{\co}{\ts+1})\iffunc(\ifsCT{\co}{\ts+1}|\dsetCT{\co}{\ts+1})\right]\Pr(\mflCT{\co}{\ts+1}|\mfCT{\co}{\ts},\jaGT{\co}{\ts})[\Gamma_{IO}(\jpolG{\co}^{t:h-1},\jaGT{\co}{\ts},\joGT{\co}{\ts+1})](\mfCT{\co}{\ts+1})\\
= & \max_{\dsetCT{\co}{\ts+1}}\sum_{\joT{\ts+1}}\sum_{\ifsCT{\co}{\ts+1}}\iffunc(\ifsCT{\co}{\ts+1}|\dsetCT{\co}{\ts+1})\sum_{\mfCT{\co}{\ts+1}}\bar{O}(\joT{\ts+1}|\jaT{\ts},\mfCT{\co}{\ts+1})\Pr(\mfnAT i{\ts+1}|\mfCT{\co}{\ts},\jaGT{\co}{\ts},\ifsCT{\co}{\ts+1})\Pr(\mflCT{\co}{\ts+1}|\mfCT{\co}{\ts},\jaGT{\co}{\ts})[\Gamma_{IO}(\jpolG{\co}^{t:h-1},\jaGT{\co}{\ts},\joGT{\co}{\ts+1})](\mfCT{\co}{\ts+1})\\
= & \max_{\dsetCT{\co}{\ts+1}}\sum_{\joT{\ts+1}}\sum_{\ifsCT{\co}{\ts+1}}\iffunc(\ifsCT{\co}{\ts+1}|\dsetCT{\co}{\ts+1})\left[\sum_{\mfCT{\co}{\ts+1}}\bar{O}(\joT{\ts+1}|\jaT{\ts},\mfCT{\co}{\ts+1})\Pr(\mfnAT i{\ts+1}|\mfCT{\co}{\ts},\jaGT{\co}{\ts},\ifsCT{\co}{\ts+1})\Pr(\mflCT{\co}{\ts+1}|\mfCT{\co}{\ts},\jaGT{\co}{\ts})[\Gamma_{IO}(\jpolG{\co}^{t:h-1},\jaGT{\co}{\ts},\joGT{\co}{\ts+1})](\mfCT{\co}{\ts+1})\right]\\
\leq & \text{\{max. of a function is greater than its expectation:\}}\\
 & \max_{\dsetCT{\co}{\ts+1}}\sum_{\joT{\ts+1}}\max_{\ifsCT{\co}{\ts+1}}\left[\sum_{\mfCT{\co}{\ts+1}}\bar{O}(\joT{\ts+1}|\jaT{\ts},\mfCT{\co}{\ts+1})\Pr(\mfnAT i{\ts+1}|\mfCT{\co}{\ts},\jaGT{\co}{\ts},\ifsCT{\co}{\ts+1})\Pr(\mflCT{\co}{\ts+1}|\mfCT{\co}{\ts},\jaGT{\co}{\ts})[\Gamma_{IO}(\jpolG{\co}^{t:h-1},\jaGT{\co}{\ts},\joGT{\co}{\ts+1})](\mfCT{\co}{\ts+1})\right]\\
= & \text{\{no dependence on \ensuremath{\dsetCT{\co}{\ts+1}}anymore:\}}\\
 & \sum_{\joT{\ts+1}}\max_{\ifsCT{\co}{\ts+1}}\sum_{\mfCT{\co}{\ts+1}}\bar{O}(\joT{\ts+1}|\jaT{\ts},\mfCT{\co}{\ts+1})\Pr(\mfnAT i{\ts+1}|\mfCT{\co}{\ts},\jaGT{\co}{\ts},\ifsCT{\co}{\ts+1})\Pr(\mflCT{\co}{\ts+1}|\mfCT{\co}{\ts},\jaGT{\co}{\ts})[\Gamma_{IO}(\jpolG{\co}^{t:h-1},\jaGT{\co}{\ts},\joGT{\co}{\ts+1})](\mfCT{\co}{\ts+1})
\end{align*}
which is the r.h.s. of \eqref{eq:IS:to-proof_2}, the inequality the
we needed to demonstrate, thereby finishing the proof.\end{proof}

\twocolumn
\end{document}